\documentclass[letterpaper,11pt]{article}
\usepackage[margin=1in]{geometry}  % set the margins to 1in on all sides

\usepackage{bbm}
\usepackage{graphicx}
\usepackage{amsmath,amssymb,amsthm,amsfonts}

\usepackage{paralist}
\usepackage{bm}
\usepackage{xspace}
\usepackage{url}
\usepackage{prettyref}
\usepackage{boxedminipage}
\usepackage{wrapfig}
\usepackage{ifthen}
\usepackage{color}
\usepackage{xspace}

\usepackage{amsmath,amsthm,amsfonts,amssymb}
\usepackage{mathtools}
\usepackage{graphicx}

\usepackage{nicefrac}

\newtheorem*{definition*}{Definition}

\DeclareMathOperator*{\argmax}{arg\,max}
\DeclareMathOperator*{\argmin}{arg\,min}
\usepackage{subcaption}

\usepackage[utf8]{inputenc}

\usepackage{xcolor}
\definecolor{expert}{HTML}{008000}
\definecolor{error}{HTML}{f96565}

\usepackage{color-edits}
\addauthor{sw}{blue}

\usepackage{thmtools}
\usepackage{thm-restate}

\usepackage{tikz}
\usetikzlibrary{arrows,calc} 
\newcommand{\tikzAngleOfLine}{\tikz@AngleOfLine}
\def\tikz@AngleOfLine(#1)(#2)#3{%
\pgfmathanglebetweenpoints{%
\pgfpointanchor{#1}{center}}{%
\pgfpointanchor{#2}{center}}
\pgfmathsetmacro{#3}{\pgfmathresult}%
}

\declaretheoremstyle[
    headfont=\normalfont\bfseries, 
    bodyfont = \normalfont\itshape]{mystyle}

\usepackage[linesnumbered,algoruled,boxed,lined,noend]{algorithm2e}

\usepackage{listings}
\usepackage{amsmath}
\usepackage{amsthm}
\usepackage{tikz}
\usepackage{caption}
\usepackage{mdwmath}
\usepackage{multirow}
\usepackage{mdwtab}
\usepackage{eqparbox}
\usepackage{multicol}
\usepackage{amsfonts}
\usepackage{tikz}
\usepackage{multirow,bigstrut,threeparttable}
\usepackage{amsthm}
\usepackage{bbm}
\usepackage{epstopdf}
\usepackage{mdwmath}
\usepackage{mdwtab}
\usepackage{eqparbox}
\usetikzlibrary{topaths,calc}
\usepackage{latexsym}
\usepackage{cite}
\usepackage{amssymb}
\usepackage{bm}
\usepackage{amssymb}
\usepackage{graphicx}
\usepackage{mathrsfs}
\usepackage{epsfig}
\usepackage{psfrag}
\usepackage{setspace}
\usepackage[%dvips,
            CJKbookmarks=true,
            bookmarksnumbered=true,
            bookmarksopen=true,
            colorlinks=true,
            citecolor=red,
            linkcolor=blue,
            anchorcolor=red,
            urlcolor=blue
            ]{hyperref}
\usepackage[linesnumbered,algoruled,boxed,lined]{algorithm2e}
\usepackage{algpseudocode}
\usepackage{stfloats}
\RequirePackage[numbers]{natbib}

\usepackage{comment}
\usepackage{mathtools}
\usepackage{blkarray}
\usepackage{multirow,bigdelim,dcolumn,booktabs}

\usepackage{xparse}
\usepackage{tikz}
\usetikzlibrary{calc}
\usetikzlibrary{decorations.pathreplacing,matrix,positioning}

\usepackage[T1]{fontenc}
\usepackage[utf8]{inputenc}
\usepackage{mathtools}
\usepackage{blkarray, bigstrut}
\usepackage{gauss}

\newcommand*{\BraceAmplitude}{0.4em}%
\newcommand*{\VerticalOffset}{0.5ex}%  
\newcommand*{\HorizontalOffset}{0.0em}% 
\newcommand*{\blocktextwid}{3.0cm}%
\NewDocumentCommand{\InsertLeftBrace}{%
	O{} % #1 = draw options
	O{\HorizontalOffset,\VerticalOffset} % #2 = optional brace shift options
	O{\blocktextwid} % #3 = optional text width
	m   % #4 = top tikzmark
	m   % #5 = bottom tikzmark
	m   % #6 = node text
}{%
	\begin{tikzpicture}[overlay,remember picture]
	\coordinate (Brace Top)    at ($(#4.north) + (#2)$);
	\coordinate (Brace Bottom) at ($(#5.south) + (#2)$);
	\draw [decoration={brace, amplitude=\BraceAmplitude}, decorate, thick, draw=black, #1]
	(Brace Bottom) -- (Brace Top) 
	node [pos=0.5, anchor=east, align=left, text width=#3, color=black, xshift=\BraceAmplitude] {#6};
	\end{tikzpicture}%
}%
\NewDocumentCommand{\InsertRightBrace}{%
	O{} % #1 = draw options
	O{\HorizontalOffset,\VerticalOffset} % #2 = optional brace shift options
	O{\blocktextwid} % #3 = optional text width
	m   % #4 = top tikzmark
	m   % #5 = bottom tikzmark
	m   % #6 = node text
}{%
	\begin{tikzpicture}[overlay,remember picture]
	\coordinate (Brace Top)    at ($(#4.north) + (#2)$);
	\coordinate (Brace Bottom) at ($(#5.south) + (#2)$);
	\draw [decoration={brace, amplitude=\BraceAmplitude}, decorate, thick, draw=black, #1]
	(Brace Top) -- (Brace Bottom) 
	node [pos=0.5, anchor=west, align=left, text width=#3, color=black, xshift=\BraceAmplitude] {#6};
	\end{tikzpicture}%
}%
\NewDocumentCommand{\InsertTopBrace}{%
	O{} % #1 = draw options
	O{\HorizontalOffset,\VerticalOffset} % #2 = optional brace shift options
	O{\blocktextwid} % #3 = optional text width
	m   % #4 = top tikzmark
	m   % #5 = bottom tikzmark
	m   % #6 = node text
}{%
	\begin{tikzpicture}[overlay,remember picture]
	\coordinate (Brace Top)    at ($(#4.west) + (#2)$);
	\coordinate (Brace Bottom) at ($(#5.east) + (#2)$);
	\draw [decoration={brace, amplitude=\BraceAmplitude}, decorate, thick, draw=black, #1]
	(Brace Top) -- (Brace Bottom) 
	node [pos=0.5, anchor=south, align=left, text width=#3, color=black, xshift=\BraceAmplitude] {#6};
	\end{tikzpicture}%
}%

\usetikzlibrary{patterns}

\definecolor{cof}{RGB}{219,144,71}
\definecolor{pur}{RGB}{186,146,162}
\definecolor{greeo}{RGB}{91,173,69}
\definecolor{greet}{RGB}{52,111,72}

\theoremstyle{plain}
\newtheorem{theorem}{Theorem}

\newtheorem{lemma}{Lemma}
\newtheorem{remark}{Remark}
\newtheorem{corollary}{Corollary}

\def\1{\mathbbm{1}}

\newenvironment{keywords}
{\bgroup\leftskip 20pt\rightskip 20pt \small\noindent{\bfseries
Keywords:} \ignorespaces}%
{\par\egroup\vskip 0.25ex}
\newlength\aftertitskip     \newlength\beforetitskip
\newlength\interauthorskip  \newlength\aftermaketitskip

\usepackage{xspace}

\newcommand{\stepa}[1]{\overset{\rm (a)}{#1}}
\newcommand{\stepb}[1]{\overset{\rm (b)}{#1}}
\newcommand{\stepc}[1]{\overset{\rm (c)}{#1}}
\newcommand{\stepd}[1]{\overset{\rm (d)}{#1}}

\newcommand{\ceil}[1]{{\left\lceil {#1} \right \rceil}}

\definecolor{myblue}{rgb}{.8, .8, 1}
\definecolor{mathblue}{rgb}{0.2472, 0.24, 0.6} % mathematica's Color[1, 1--3]
\definecolor{mathred}{rgb}{0.6, 0.24, 0.442893}
\definecolor{mathyellow}{rgb}{0.6, 0.547014, 0.24}

\usepackage{cleveref}
\crefname{lemma}{Lemma}{Lemmas}
\Crefname{lemma}{Lemma}{Lemmas}
\crefname{thm}{Theorem}{Theorems}
\Crefname{thm}{Theorem}{Theorems}
\Crefname{assumption}{Assumption}{Assumptions}
\Crefname{inftheorem}{Informal Theorem}{Informal Theorems}
\crefformat{equation}{(#2#1#3)}

\newcommand{\Acal}{\mathcal{A}}

\newcommand {\Fcal} {{\mathcal{F}}}
\newcommand{\indic}{\mathbbm{1}}

\newcommand{\br}{\bar{r}}

\newcommand {\E} {{\mathbb{E}}}
\newcommand {\R} {{\mathbb{R}}}
\newcommand{\htheta}{{\widehat{\theta}}}

\newcommand{\reg}{\mathsf{R}}

\newcommand{\Nout}{N_{\mathrm{out}}}

\newcommand{\Aact}{\mathcal{A}_{\mathrm{act}}}
\newcommand{\Acon}{\mathcal{A}_{\mathrm{con}}}

\newcommand\numberthis{\addtocounter{equation}
{1}\tag{\theequation}}

\DeclarePairedDelimiter{\abs}{\lvert}{\rvert}

\DeclarePairedDelimiter{\parr}{(}{)}
\DeclarePairedDelimiter{\parq}{[}{]}

\DeclarePairedDelimiter{\bra}{\lbrace}{\rbrace}

\usepackage{cleveref}
\crefname{lemma}{Lemma}{Lemmas}
\Crefname{lemma}{Lemma}{Lemmas}
\crefname{thm}{Theorem}{Theorems}
\Crefname{thm}{Theorem}{Theorems}
\crefname{remark}{Remark}{Remarks}
\Crefname{remark}{Remark}{Remarks}
\Crefname{algorithm}{Algorithm}{Algorithms}
\crefname{algorithm}{Algorithm}{Algorithms}
\Crefname{section}{Section}{Sections}
\crefname{section}{Section}{Sections}
\Crefname{assumption}{Assumption}{Assumptions}
\Crefname{inftheorem}{Informal Theorem}{Informal Theorems}
\crefformat{equation}{(#2#1#3)}

\numberwithin{theorem}{section}

\begin{document}
%\begin{frontmatter}

% "Title of the paper"
\title{Optimal Arm Elimination Algorithms for Combinatorial Bandits}

\author{Yuxiao Wen, Yanjun Han, Zhengyuan Zhou\thanks{Yuxiao Wen is with the Courant Institute of Mathematical Sciences, New York University, email: \url{yuxiaowen@nyu.edu}. Yanjun Han is with the Courant Institute of Mathematical Sciences and the Center for Data Science, New York University, email: \url{yanjunhan@nyu.edu}. Zhengyuan Zhou is with the Stern School of Business, New York University, and Arena Technologies, email: \url{zz26@stern.nyu.edu}.}}

\maketitle
%\thankstext{T1}{Footnote to the title with the ``thankstext'' command.}

\begin{abstract}%
Combinatorial bandits extend the classical bandit framework to settings where the learner selects multiple arms in each round, motivated by applications such as online recommendation and assortment optimization. While extensions of upper confidence bound (UCB) algorithms arise naturally in this context, adapting arm elimination methods has proved more challenging. We introduce a novel elimination scheme that partitions arms into three categories (confirmed, active, and eliminated), and incorporates explicit exploration to update these sets. We demonstrate the efficacy of our algorithm in two settings: the combinatorial multi-armed bandit with general graph feedback, and the combinatorial linear contextual bandit. In both cases, our approach achieves near-optimal regret, whereas UCB-based methods can provably fail due to insufficient explicit exploration. Matching lower bounds are also provided.
\end{abstract}


\section{Introduction}\label{sec:intro}
Combinatorial bandits, where the learner picks a subset of actions rather than a single action, have a wide range of applications spanning online recommendations \citep{wang2017efficient,qin2014contextual}, assortment optimization \citep{han2021adversarial}, crowd-sourcing \citep{lin2014combinatorial}, webpage optimization \citep{liu2021map}, online routing \citep{gyorgy2007line,audibert2014regret}, etc. %\YH{Please add examples and references.}
A common feedback model in this setting is the \emph{semi-bandit feedback}, in which the learner observes the rewards of the chosen actions only. %\YH{Throughout, let's call it ``combinatorial bandit'', but the feedback as ``semi-bandit feedback''.} 
Achieving optimal performance under this limited feedback requires carefully balancing exploration and exploitation: the learner must search for promising new actions while simultaneously committing to actions already known to perform well.

In stochastic bandits, where rewards follow fixed but unknown distributions, two prominent algorithmic frameworks achieve the optimal exploration-exploitation tradeoff: the \emph{upper confidence bound} (UCB) algorithm \citep{auer2002finite} and the \emph{arm elimination} algorithm \citep{even2006action}. Both approaches maintain confidence intervals for the rewards of each action, but proceed differently: UCB selects the action with the largest upper bound, whereas arm elimination retains only the ``active'' actions whose confidence intervals overlap with that of the action with the highest lower bound. Consequently, UCB relies on \emph{implicit} exploration, while arm elimination can perform more \emph{explicit} exploration over the active set. For both multi-armed and linear bandits, these algorithms are known to achieve near-optimal regret \citep{auer2010ucb,abbasi2011improved}. %\YH{Add reference.}

For more complex bandit problems, the lack of explicit exploration can make UCB less effective than arm elimination methods. One example is the multi-armed bandit with graph feedback, where exploration must be carefully guided by the graph structure. Another is the linear contextual bandit with a finite set of time-varying contexts, where dependencies across rewards undermine the validity of confidence bounds. In both cases, arm elimination provides a remedy. For bandits with graph feedback, arm elimination explicitly leverages the graph structure to explore the active set of actions \citep{han2024optimal}. For linear contextual bandits, a master algorithm built on top of an arm elimination subroutine can effectively handle the dependence and achieve the optimal regret \citep{chu2011contextual}. 

On the other hand, while the UCB algorithm admits a relatively direct extension to combinatorial bandits \citep{kveton2015tight,combes2015combinatorial}, arm elimination methods in this setting remain largely underexplored. The main difficulty lies in balancing exploration of new arms with exploitation of known good arms within the same round. In this paper, we study two instances of combinatorial bandits:
\begin{enumerate}
\item Combinatorial multi-armed bandits with general graph feedback, where chosen actions can reveal the rewards of other arms according to a given graph structure.
\item Combinatorial linear contextual bandits, where the learner selects a subset of contexts under a linear reward model.
\end{enumerate}
For both problems, existing UCB and arm elimination methods fail to achieve optimal regret, either due to the lack of explicit exploration (UCB) or the challenges posed by the combinatorial structure (arm elimination). To address this, we propose a general arm elimination framework for combinatorial bandits and show that it achieves optimal regret in both settings.

\subsection{Our Contributions}
The main contributions of this paper are as follows:
\begin{enumerate}
    \item We propose a general arm elimination framework for combinatorial bandits that partitions actions into three categories (confirmed, active, and eliminated). At each round, the combinatorial budget is allocated between confirmed and active actions, with explicit exploration directed toward the active set. 
    \item For combinatorial bandits with general graph feedback, we design an arm elimination algorithm based on this framework that achieves near-optimal regret. Specifically, with time horizon $T$ and combinatorial budget $S$, and for feedback graphs with independence number $\alpha$, the algorithm simultaneously attains the optimal worst-case regret $\widetilde{\Theta}(\sqrt{\alpha ST}+S\sqrt{T})$ and the optimal gap-dependent regret $\widetilde{\Theta}(\frac{\alpha+S}{\Delta_*}\log(T))$. 
    \item For combinatorial linear contextual bandits with dimension $d$ and a finite number of contexts, we show that combining our arm elimination method with the master algorithm of \cite{auer2002using} yields the optimal regret $\widetilde{\Theta}(\sqrt{dST})$.
\end{enumerate}

\subsection{Related Work}

\paragraph{Combinatorial Bandits.}
The minimax regret for combinatorial bandits under graph feedback is recently shown to be $\widetilde{\Theta}(S\sqrt{T}+\sqrt{\alpha ST})$ even in the adversarial setting \citep{wen2025adversarial}. In the stochastic regime, \cite{kveton2015tight} provides an instance-dependent regret bound $O\parr*{\frac{KS\log(T)}{\Delta_*}}$ for a UCB-type algorithm and shows that this rate is tight when the feasible decision set is some constrained subset of the set $\binom{[K]}{S}$. Under the unconstrained decision set $\binom{[K]}{S}$, the bound is later improved to $O\parr*{\frac{K\sqrt{S}\log(T)}{\Delta_*}+\frac{KS^3}{\Delta_*^2}}$ by \cite{combes2015combinatorial}. Their result shaves a factor $\sqrt{S}$ but has an extra term that dominates for small $\Delta_*$. Both of their algorithms are described by the UCB in \cref{alg:ucb} with slightly different input parameter $L$. Later, \cite{wang2018thompson} closes the gap in the logarithmic term. They show that a variant of Thompson Sampling achieves regret $O\parr*{\frac{K\log(T)}{\Delta_*} + \frac{S}{\Delta_*^4}}$. However, for small $\Delta_*$ this extra term again dominates and leads to a loose upper bound. Under the full-information feedback when the rewards of all actions are revealed, no instance-dependent bound is shown to the best of our knowledge. 
% \YH{Use footnotes sparingly. You should take a decisive stand on what you want to emphasize and what you can omit. Minimize middlegrounds.}

\paragraph{Combinatorial Linear Contextual Bandits.}
Another line of research concerns the setting where contextual information is available to the learner to aid decision-making. A widely adopted (stochastic) reward model assumes that the expected reward is linear in the observed context, whereas no assumption is imposed on how the context is generated. This formulation finds many industrial applications, such as recommender systems \citep{qin2014contextual} and assortment management \citep{han2021adversarial}. In terms of regret, \cite{qin2014contextual} presents a variant of \texttt{LinUCB} in \cite{li2010contextual} that achieves $\widetilde{O}(d\sqrt{ST})$ regret. In comparison, in the classical bandit setting with $S=1$, the near-optimal regret is known to be $\widetilde{O}(\sqrt{dT})$ and is achieved by an elimination-based algorithm \citep{chu2011contextual} paired with a master algorithm to handle dependence. The best known result for UCB-type algorithms is $\widetilde{O}(d\sqrt{T})$ by \cite{abbasi2011improved}.

\paragraph{Elimination-based Bandit Algorithms.}
The idea of arm elimination in bandit algorithms originates from \cite{even2006action}. Compared to the arguably more natural UCB-type algorithms, elimination has a demonstrated value in a range of bandit problems, including MAB with graph feedback and contextual MAB. In the former, elimination allows the learner to force exploration as the algorithm runs to achieve the optimal trade-off under general graphs \citep{han2024optimal}. In the linear contextual case, \cite{auer2002using} develops a hierarchical elimination scheme to address a dependence issue in the reward observations; this scheme has been widely adopted in various settings with contextual information to achieve tight regret \citep{chu2011contextual,han2024optimal,wen2025adversarial}. Nonetheless, despite its power in the bandit problems, it remains unclear how to perform elimination, or whether it is possible at all, under the combinatorial setting where the learner selects and compares to $S>1$ optimal arms. % A detailed discussion is left to \cref{sec:elimination_alg_explained}.

\paragraph{Top-$S$ Arm Identification.}
Another relevant line of research is top-$S$ arm identification. In this problem, the learner pulls one arm at a time, and the objective is to identify the $S$ arms with maximal expected rewards, either under a fixed budget or up to a fixed confidence. For best-arm identification where $S=1$, there is rich literature studying both the complexity under a fixed confidence \citep{audibert2010best,garivier2016optimal} and the probability of error under a fixed budget \citep{carpentier2016tight,kato2022best,komiyama2022minimax}. For general $S\ge 1$, \cite{chen2014combinatorial} tackles this problem using UCB under a constrained subset, i.e. the identified $S$ arms must belong to a certain subset. \cite{chen2017nearly,zhou2022approximate} address the unconstrained problem via elimination-based algorithms. 
Nonetheless, this arm identification problem does not face the core exploration-exploitation trade-off \emph{within the same round} in combinatorial bandits. 

\subsection{Notations}
For $n\in\mathbb{N}$, let $[n] = \bra{1,2,\dots,n}$. For a directed graph $G=(V,E)$, let $\Nout(a) = \bra{b\in V: (a,b)\in E}$ denote the out-neighbors of node $a\in V$ and $\Nout(U) = \cup_{a\in U}\Nout(a)$ denote the out-neighbors of a subset $U$. For a vector $x\in\R^d$ a PSD matrix $A\in\R^{d\times d}$, the matrix norm is defined by $\|x\|_A = \sqrt{x^\top Ax}$. We use $\widetilde{O}$ to denote the usual asymptotic meaning of $O$ but suppress less important poly-logarithmic factors.

\section{Combinatorial Bandits with Graph Feedback}\label{sec:graph_feedback}

\subsection{Problem Formulation}\label{sec:graph_feedback_intro}

This section introduces the problem of combinatorial bandits with general graph feedback. At each time $t$ over a horizon of length $T$, the learner selects a \textit{decision} that is a subset of arms $V_t\subseteq [K]$ %\YH{I would prefer the notation $S_t$. At least sets should be capitalized.} 
such that $|V_t| = S$ for a fixed $S\ge 1$. There is a known directed feedback graph $G=([K],E)$ over the arms. The learner observes the individual rewards $\bra{r_{t,a}: a\in\Nout(V_t)}$ and receives the total reward $r_{t,V_t} = \sum_{a\in V_t}r_{t,a}$. For the scope of this work, we assume $G$ contains all self-loops, i.e. $a\in\Nout(a)$. We assume $r_{t,a}\in[0,1]$ and for each arm $a$, the rewards $\bra{r_{t,a}}_{t\in[T]}$ are i.i.d. with a time-invariant mean $\mu_a$. 

Without loss of generality (WLOG), we assume the means satisfy $\mu_1 \ge \mu_2 \ge \cdots \ge \mu_K$. For any policy $\pi$, the regret measures the expected loss compared to the hindsight optimal decision $[S]$:
\[
\reg(\pi) = \sum_{t=1}^T\bigg(\sum_{i=1}^S\mu_i - \sum_{a\in V_t}\mu_a\bigg). 
\]
% where $V_t$ is the decision selected by the learner's policy $\pi$ at time $t$.

\subsection{An Elimination-based Algorithm}\label{sec:elimination_alg_explained}

We start by giving a high-level intuition of the arm elimination algorithm under the classical setting $S=1$ \citep{even2006action}. In this case, the algorithm maintains an \textit{active set} of ``probably good'' arms $\Aact\subseteq [K]$ and a minimum count $N=\min_{a\in\Aact}n_{t,a}$ for the active arms, where $n_{t,a}$ denotes the number of observations for arm $a$ up to time $t$. It then uniformly explores every arm in $\Aact$ with a small $n_{t,a}$ and update $N$ accordingly. By standard concentration results (see \cref{lem:sto_reward_concentration} below), the algorithm recognizes a uniform confidence width for each $\mu_a$ and eliminates any arm $a$ from the set $\Aact$ that is provably suboptimal based on the confidence widths.

\begin{lemma}\label{lem:sto_reward_concentration}
Fix any $\delta\in(0,1)$. With probability at least $1-\delta$, we have for every arm $a$ at every time $t$,
\[
\abs*{\br_{t,a} - \mu_a} \le  \sqrt{\log\parr*{2KT/\delta}/n_{t,a}} \eqqcolon w(n_{t,a})
\]
where $\br_{t,a}$ is the empirical mean and $n_{t,a}$ the number of observations at time $t$.
\end{lemma}

However, it is challenging to extend this elimination scheme to $S\ge 2$, as it is unclear how to set the elimination benchmark. When $S=1$, we only need to decide if an arm $a$ is possibly optimal or not, so the algorithm simply eliminates based on the empirically best arm. When $S\ge 2$, an arm $a$ can be worse than $\mu_1$ but still possible to be the $S$-th optimal, rendering this choice impractical. Even if we manage to eliminate arms that are provably worse than the $S$-th optimal arm for any decision $V_t\subseteq \Aact$, there is no guarantee that $\sum_{i=1}^S\mu_i - \sum_{a\in V_t}\mu_a$ will be small, because $\mu_a$ may be close to $\mu_S$ but much worse than $\mu_{S-1}$ or $\mu_1$.

To address the aforementioned challenges, we introduce the \emph{confirmed set} $\Acon$ in our \cref{alg:comb_arm_elim_full}. Suppose the confidence width $w_t$ for every active arm $a\in\Aact$ is uniformly positive at time $t$. The idea of confirmation is to use this width $w_t$ to identify a subset $\Acon\subseteq [S]$. Each arm $i\in\Acon$ has a mean $\mu_i\gg \mu_S+w_t$ that is much larger than the mean of the $S$-th optimal arm. Therefore, the learner must have $\Acon\subseteq V_t$ to avoid incurring an instantaneous regret unbounded by $w_t$. In addition, for every unconfirmed optimal arm $i'\in[S]\backslash\Acon$, it holds that $\mu_{i'}\lesssim \mu_S + w_t$, allowing the learner to include any active arm $a\in\Aact$ in the decision $V_t$ and suffer a bounded regret, even if the learner only manages to bound $\mu_S - \mu_a$.

Note that by partitioning the uneliminated arms into $\Acon$ and $\Aact\backslash\Acon$, we are effectively identifying an \textit{exploration-exploitation trade-off} based on the confidence width $w_t$ at every time $t$. The confirmed arms $\Acon$ are ``too good to leave out'' and lead to an exploitation of size $|\Acon|$. Meanwhile, we use the remaining $S-|\Acon|$ budget in our decision $V_t$ to explore the remaining active arms and further eliminate the suboptimal ones. For the graph feedback, we adopt the exploration strategy in \cite{han2024optimal} and successively pull the arm with the largest out-degree (Line 6 of \cref{alg:comb_arm_elim_full}). This exploration budget $S-|\Acon|$, notably, changes over time. % as a consequence of $w_t$ at every time $t$.

% For the sake of simplicity, we write the confidence width as $w(n)$ in the elimination step in \cref{alg:comb_arm_elim_full}. 
The overall algorithm is given in \cref{alg:comb_arm_elim_full}. Recall that $w(n)\coloneqq \sqrt{\log(2KT/\delta)/n}$ is given in \cref{lem:sto_reward_concentration}.

\begin{algorithm}[!ht]\caption{Combinatorial Arm Elimination}
\label{alg:comb_arm_elim_full}
\textbf{Input:} failure probability $\delta\in(0,1)$.

\textbf{Initialize:} Confirmed set $\Acon\gets \varnothing$, active set $\Aact\gets [K]$, and minimum count $N\gets 0$.

\For{$t=1$ \KwTo $T$}{
Let $\Acal_0\gets \bra*{a\in \Aact: n_{t-1,a} = N}$.

\For{$j=1$ \KwTo $S-\abs{\Acon}$}{
Let $a_{t,j}\in\Acal_0$ be any arm with the largest out-degree in $G|_{\Acal_0}$.

Update $\Acal_0\gets \Acal_0\backslash \Nout(a_{t,j})$.
}

Assemble $V_t\gets \Acon\cup\bra{a_{t,j}}_{j=1}^{S-\abs{\Acon}}$.

Observe feedback $\bra*{r_{t,a}: a\in\Nout(V_t)}$.

Update $(\br_{t,a}, n_{t,a})$ as the average reward and observation count of arm $a$ by the end of time $t$.

\If{$\min_{a\in\Aact}n_{t,a} > N$}{
Update count $N\gets \min_{a\in\Aact}n_{t,a}$.

Let $\br_{t,(S)}$ be the $S$-th empirically best reward in the union set $\Acon\cup\Aact$.

Update the confirmed set $\Acon$ to be:
\begin{equation*}\label{eq:elim_step_confirmed}
\Acon \cup \bra*{a\in\Aact : \br_{t,a} > \br_{t,(S)} + 4w(N)}.
\end{equation*}
Then the active set $\Aact$ to be
\begin{equation*}\label{eq:elim_step_active}
\bra*{a\in\Aact\backslash\Acon : \br_{t,a} \geq \br_{t,(S)} - 2w(N)}.
\end{equation*}
}
}
\end{algorithm}

\subsection{Algorithmic Properties}\label{sec:graph_alg_properties}

We now present several key properties of \cref{alg:comb_arm_elim_full}. For the sake of clarity, we use $N^t$, $\Aact^t$, and $\Acon^t$ to denote the minimum count of each active arm's observations $N$, the active set $\Aact$, and the confirmed set $\Acon$ in \cref{alg:comb_arm_elim_full} by the end of time $t$. Let $\br_{t,(j)}$ denote the $j$-th empirically best reward in $\Acon^t\cup\Aact^t$, and $a_{t,(j)}$ denote the corresponding arm.\footnote{For notational simplicity, here we use $(j)$ to denote the $j$-th largest as opposed to the $j$-th smallest in the conventional notations for order statistics.} For $i\in[S]$ and $a\in[K]$, let $\Delta_{a,i} = \mu_i - \mu_a$ denote the reward gap between arm $a$ and the $i$-th optimal arm. WLOG, suppose $\Delta_*\triangleq \Delta_{S+1,S} > 0$, which serves as a margin and characterizes the difficulty of distinguishing suboptimal and optimal arms.\footnote{If $\mu_S=\mu_{S+1}=\cdots=\mu_{S+k}$, it is straightforward to extend our analysis to the definition $\Delta_*=\Delta_{S,S+k+1}$.} 

% Note that $\Acon^{t-1}\subseteq V_t$. Importantly, the number of observations for each confirmed arm in $\Acon^t$ is always at least the minimum count $N^t$, although $N^t$ is updated using the active set $\Aact^t$.

Conditioned on the validity of the confidence width in \cref{lem:sto_reward_concentration}, the following lemma lists three important properties for \cref{alg:comb_arm_elim_full}.

\begin{lemma}\label{lem:alg_properties_graph}
Suppose the event in \cref{lem:sto_reward_concentration} holds. Then for each time $t\in[T]$, the following events hold by the end of $t$:
\begin{enumerate}[(A)]
    \item The optimal arms remain uneliminated, i.e. $[S]\subseteq \Aact^t\cup\Acon^t$.

    \item $\Acon^t\subseteq [S-1]$.

    \item Let $i_{*,t}=\min\Aact^t$. For every $a\in\Aact^t$, it holds that $\Delta_{a,i_{*,t}} \le 8w(N^t)$.
\end{enumerate}
\end{lemma}

% The proof is deferred to \Cref{app:alg_properties_graph_proof}. 
The first property states that the optimal arms $[S]$ are never eliminated as the algorithm runs. This serves as the basis of the elimination scheme, since otherwise the algorithm would suffer a linear regret.

Second, any confirmed arm in $\Acon^t$ always belongs to the top $S-1$ optimal arms. Recall that $\Acon^{t-1}\subseteq V_t$ for all $t\in[T]$. This optimality ensures that pulling the confirmed arms incurs no instantaneous regret, since now $\Acon^{t-1}\subseteq [S]\cap V_t$ for every time $t$. Therefore,
\begin{align*}
\sum_{i=1}^S\mu_i - \sum_{a\in V_t}\mu_a &= \sum_{i\in[S]\backslash\Acon^{t-1}}\mu_i - \sum_{a\in V_t\backslash\Acon^{t-1}}\mu_a
\le \sum_{a\in V_t\backslash\Acon^{t-1}}\parr*{\mu_{i_{*,t}} - \mu_a}\numberthis\label{eq:inst_reg_cancellation}
\end{align*}
which follows from that $i_{*,t}=\min \Aact^t$ and so $\mu_i \le \mu_{i_{*,t}}$ for every $i\in[S]\backslash\Aact^t$. The inequality \cref{eq:inst_reg_cancellation} serves as the basis to derive regret bounds. 
% \YW{Add inequality in terms of $i_{*,t}$ and proof sketch.}

The last claim concerns the optimal active arm $i_{*,t}$ at each time $t$ in hindsight. Note that by the second claim, $\Acon^t\subseteq [S-1]$ and so $i_{*,t}\in[S]$. The intuition is as follows. In the elimination step (last line) in \cref{alg:comb_arm_elim_full}, the benchmark is the $S$-th empirically best reward $\br_{t,(S)}$. As a result, this choice guarantees that $\Delta_{a,S}=O(w(N^t))$ for any active arm $a\in\Aact^t$. Additionally, the update criterion of the confirmed set implies that $\mu_{i_{*,t}}-\mu_S = O(w(N^t))$ because $i_{*,t}\notin\Acon^t$. Together they give the claimed bound on $\Delta_{a,i_{*,t}}$ and therefore an upper bound of \eqref{eq:inst_reg_cancellation}.

Intuitively, at time $t$, the learning process only distinguishes the eliminated and the active arms up to the confidence width $w(N^t)$, where $N^t$ lower bounds the number of observations for all active arms. The instantaneous regret at $t$ is inevitably $O(w(N^t))$. Therefore, at the current time, we may treat any arm $i\in[S-1]$ and $S$ indistinguishably if $\mu_i-\mu_S=O(w(N^t))$. For any arm $i\in[S-1]$ beyond this width, it becomes necessary to identify and include $i$ in the final decision $V_t$, which motivates the choice of the confirmed set $\Acon^t$.

\subsection{Regret Bounds}
In this section, we present the regret guarantees for \cref{alg:comb_arm_elim_full} and show their tightness. Specifically, our algorithm simultaneously achieves a logarithmic gap-dependent bound $O(\log(KT)(\alpha\log^2 K+S)/\Delta_*)$ and a worst-case regret bound $\widetilde{O}(\sqrt{\alpha ST} + S\sqrt{T})$.

\begin{theorem}[Instance-dependent regret]\label{thm:gap_dependent_reg_upper_bound}
Fix any $\delta\in(0,1)$. With probability at least $1-\delta$, 
\[
\reg(\mathrm{Alg}~\ref{alg:comb_arm_elim_full}) = O\parr*{\log(2KT/\delta)\frac{\alpha \log^2 K + S}{\Delta_*}}.
\]
\end{theorem}

Specializing to the semi-bandit feedback ($\alpha=K$) and full-information feedback ($\alpha=1$), we obtain improved or new regret bounds for combinatorial bandits. Note that we do not have remainder terms (such as $O(\Delta_*^{-4})$ in \cite{wang2018thompson}) in \Cref{thm:gap_dependent_reg_upper_bound}.  

% Since the general graph feedback subsumes the semi-bandit feedback with $\alpha=K$, we immediately obtain an improved and near-optimal instance-dependent bound. In particular, our bound removes the extra terms that scale with $\Delta_*^{-4}$ in the previous results and becomes tight for all range of $\Delta_*$. \Cref{thm:gap_dependent_reg_upper_bound} also gives the first bound under full information.

\begin{corollary}
Under the semi-bandit feedback, we have 
\[
\reg(\mathrm{Alg}~\ref{alg:comb_arm_elim_full}) = O\parr*{\log(2KT/\delta)\frac{K\log^2 K}{\Delta_*}}.
\]
Under the full-information feedback, we have
\[
\reg(\mathrm{Alg}~\ref{alg:comb_arm_elim_full}) = O\parr*{\log(2KT/\delta)\frac{\log^2 K+S}{\Delta_*}}.
\]
\end{corollary}

We also prove matching lower bounds to show that the regret in \Cref{thm:gap_dependent_reg_upper_bound} is near-optimal for all algorithms. Note that in \Cref{thm:gap_dependent_reg_lower_bound}, the missing $\log T$ factor for small $\alpha$ is not an artifact of our analysis; in fact, under a full-information feedback ($\alpha=1$), one can attain a \textit{constant} regret using a simple greedy algorithm \citep{degenne2016anytime}. 

% , due to the rich feedback. Indeed, in multi-armed bandits ($S=1$) with full information ($\alpha=1$), one can attain a \textit{constant} regret using a simple greedy algorithm \citep{degenne2016anytime}.\footnote{The constant $2$ here is not essential and can be replaced by any constant $c>1$ in the proof in \cref{app:instance_lower_bound}.}
\begin{theorem}\label{thm:gap_dependent_reg_lower_bound}
Let $\reg_\nu(\pi)$ denote the regret of policy $\pi$ under bandit environment $\nu$. Fix any policy $\pi$:
\begin{itemize}
    \item[(L1)] Suppose the policy satisfies $\max_\nu\reg_\nu(\pi) \le CT^p$, $\Delta_* \in (T^{-(1-p)},\frac{1}{4}]$, and $\alpha\ge 2S$, and for some constants $C>0$ and $p\in[0,1)$. 
    
    Then $\max_\nu\reg_\nu(\pi) = \Omega\parr*{\log(T\Delta_*)\min\bra*{\frac{\alpha}{\Delta_*},\frac{1}{\Delta_*^2}}}$.

    \item[(L2)] Suppose $K\ge 2S$ and $\Delta_*\le \frac{1}{2}$.
    
    Then $\max_\nu\reg_\nu(\pi) = \Omega\parr*{\min\bra*{\frac{S}{\Delta_*},\Delta_* ST}}$.
\end{itemize}
\end{theorem}

We can also derive a minimax regret bound for \cref{alg:comb_arm_elim_full}, which nearly matches the lower bound $\Omega(\sqrt{\alpha ST} + S\sqrt{T})$ in Theorem 1.3 of \cite{wen2025adversarial}.

\begin{theorem}[Minimax regret]\label{thm:minimax_reg_upper_bound}
Fix any $\delta\in(0,1)$. With probability at least $1-\delta$, 
\[
\reg(\mathrm{Alg}~\ref{alg:comb_arm_elim_full}) = O\parr*{\log^2 K\sqrt{\log(2TK/\delta)}\parr*{\sqrt{\alpha ST} + S\sqrt{T}}}.
\]
\end{theorem}

\subsection{Suboptimality of UCB}

As discussed in \cref{sec:intro}, a large volume of bandit literature adopts UCB-type algorithms to develop optimal regret guarantees in different settings. In combinatorial bandits, \cite{kveton2015tight} proposes a natural UCB algorithm, called \texttt{CombUCB1}, that achieves the optimal $\widetilde{O}(\sqrt{KST})$ regret under the semi-bandit feedback. However, this UCB algorithm, described in \cref{alg:ucb}, is provably suboptimal under general graph feedback.

\begin{algorithm}\caption{Combinatorial UCB}
\label{alg:ucb}
\textbf{Input:} width parameter $L>0$, and failure probability $\delta\in(0,1)$.

Let $(\br_{t,a}, n_{t,a})$ be the empirical reward and the observation count of arm $a$ at the end of $t$.

Let $\Acal = \bra{v\subseteq [K]: |v| = S}$ be the set of feasible decisions.

\For{$t=1$ \KwTo $T$}{
Select $V_t \gets \argmax_{v\in\Acal}\sum_{a\in v}\br_{t-1,a} + \frac{L}{\sqrt{n_{t-1,a}}}$.

Observe feedback $\bra*{r_{t,a}: a\in\Nout(V_t)}$.

Update $(\br_{t,a}, n_{t,a})$ accordingly.
}
\end{algorithm}

The parameter $L>0$ is any factor such that with probability at least $1-\delta$,
\[
\abs*{\br_{t,a} - \mu_a} \le \frac{L}{\sqrt{n_{t,a}}}
\]
for every arm $a\in[K]$ at every time $t\in[T]$. For instance, \cref{lem:sto_reward_concentration} gives one possible option $L=\sqrt{\log(2KT/\delta)}$.
It maintains such a UCB for each arm $a\in[K]$ and, at each time $t$, selects the combination of $S$ arms that maximizes the total UCB.

The high-level reason behind the suboptimality shown in \Cref{thm:ucb_fail} is that, without forced exploration, UCB essentially uses all $S$ arms for either exploitation or exploration \textit{simultaneously} at each time. In contrast, our elimination scheme in \cref{alg:comb_arm_elim_full} crucially relies on an exploration-exploitation separation among the $S$ arms at each time.

\begin{theorem}\label{thm:ucb_fail}
Fix any $(S,\alpha,K,T)$ with $S\alpha\le K$ and $\alpha >1$. There is a problem instance under which 
\[
\reg(\mathrm{Alg}~\ref{alg:ucb}) = \Omega(LS\sqrt{\alpha T}).
\]
\end{theorem}

% The proof is deferred to \Cref{app:suboptimal_ucb}.
\Cref{thm:ucb_fail} shows that, in the regime when $S\alpha \le K$, UCB achieves a suboptimal rate compared to $\widetilde{O}(S\sqrt{T}+\sqrt{\alpha ST})$ in \Cref{thm:minimax_reg_upper_bound}. When $S\alpha\ge K$, UCB achieves the tight rate $\widetilde{\Theta}(\sqrt{KST})$ \citep{kveton2015tight}. Interestingly, this sub-optimality ratio $\min\bra{S\alpha,K}/(S+\alpha)$ scales as $\Theta(1)$ at two extremes when $\alpha=K$ and $\alpha=1$.\footnote{To our knowledge, for the full-information feedback ($\alpha=1$) we do not find an existing regret bound for UCB. For completeness, we show 
% in \Cref{app:ucb_full_info} 
that \cref{alg:ucb} attains $\widetilde{O}(S\sqrt{T})$ regret in this case.}

\section{Combinatorial Linear Contextual Bandits}\label{sec:lin_context}

\subsection{Problem Formulation}\label{sec:lin_context_intro}
In combinatorial linear contextual bandits, each arm $a\in[K]$ is associated with a context vector $x_{t,a}\in\R^d$ at time $t$. We consider the linear model where $r_{t,a} = \theta_*^\top x_{t,a}+\varepsilon_{t,a}\in[-1,1]$ for an unknown parameter $\theta_*\in\R^d$ and mean-zero noise $\varepsilon_{t,a}$. We assume $\|x_{t,a}\|_2\le 1$ for all $t\in[T]$ and $a\in[K]$, and $\|\theta_*\|_2\le 1$. 

Given the contextual information, we consider a stronger \textit{dynamic} regret where the hindsight optimal oracle knows $\theta_*$ and chooses decisions conditioned on the context:
\[
\reg(\pi) = \sum_{t=1}^T\parr*{\max_{\substack{V_{*,t}\subseteq [K]\\|V_{*,t}|=S}}\sum_{i\in V_{*,t}}\theta_{*}^\top x_{t,i} - \sum_{a\in V_t}\theta_*^\top x_{t,a}}
\]
where the decision $V_t$ is selected by the learner's policy $\pi$ at time $t$.

\subsection{A Hierarchical Elimination Algorithm}\label{sec:hierarchical_elim_idea}

% \YH{Sets $U$ are not well-explained in the algorithm.}

In this section, we introduce an algorithm that builds on a hierarchical elimination idea. This idea stems from \cite{auer2002using} and has been widely applied in the contextual bandit literature thereafter \citep{chu2011contextual,han2024optimal,wen2025joint}. Different from the UCB analysis that bounds $\|\theta_*-\htheta_t\|$ uniformly \citep{abbasi2011improved}, it bounds the estimation error $\abs{\theta_*^\top x_{t,a}-\htheta_t^\top x_{t,a}}$ along the \textit{realized} direction $x_{t,a}$, where $\htheta_t$ is the estimated parameter.

\begin{algorithm}[h!]\caption{Hierarchical Arm Elimination}
\label{alg:master_alg}

\textbf{Initialize:} Set $H\gets \lceil \log_2(\sqrt{ST}) \rceil$, $\Phi^{(h)}_1=\varnothing$ for $h\in[H]$.

\For{$t=1$ \KwTo $T$}
{   
Observe the contexts $\bra{x_{t,a}}_{a\in[K]}$.

Initialize $A_1\leftarrow [K]$.

Initialize the decision $V_t\gets \varnothing$.

\For{stage $h=1$ \KwTo $H$}
{
Use observations in $\Phi^{(h)}_t$ and \cref{alg:base_alg} to compute reward $\widehat{r}^{(h)}_{t,a}$ and width $w^{(h)}_{t,a}$ for $a\in A_h$.

(1) Let $U_1 = \bra{a\in A_h\backslash V_t: w^{(h)}_{t,a}>2^{-h}}$.

\uIf{$|U_1|\le S-|V_t|$}{
Set $U_{t}^{(h,1)} \gets U_1$.
}\Else{
Set any $U_{t}^{(h,1)}\subseteq U_1$ with $|U_{t}^{(h,1)}|=S-|V_t|$.
}

Add $V_t\gets V_t\cup U_t^{(h,1)}$.

Let $a_{1,(S-|V_t|)}$ be the $(S-|V_t|)$-th arm maximizing $\widehat{r}^{(h)}_{t,a}+w^{(h)}_{t,a}$ in $A_h\backslash V_t$.

(2) Let $U_2=$
\[
\bra*{a\in A_h\backslash V_t: \widehat{r}^{(h)}_{t,a} > \widehat{r}^{(h)}_{t,a_{1,(S-|V_t|)}} + 4\cdot 2^{-h}}
\]

\uIf{$|U_2|\le S-|V_t|$}{
Set $U_{t}^{(h,2)} \gets U_2$.
}\Else{
Set any $U_{t}^{(h,2)}\subseteq U_2$ with $|U_{t}^{(h,2)}|=S-|V_t|$.
}

Add $V_t\gets V_t\cup U_t^{(h,2)}$.

Let $a_{2,(S-|V_t|)}$ be the $(S-|V_t|)$-th arm maximizing $\widehat{r}^{(h)}_{t,a}+w^{(h)}_{t,a}$ in $A_h\backslash V_t$.

(3) Find active set $A_{h+1}$ for next stage:
\[
\bra*{a\in A_h\backslash V_t : \widehat{r}^{(h)}_{t,a} \geq \widehat{r}^{(h)}_{t,a_{2,(S-|V_t|)}} - 2\cdot 2^{-h}}
\]

Update $\Phi^{(h)}_{t+1}\gets \Phi^{(h)}_{t}\oplus (U_t^{(h,1)}\cup U_t^{(h,2)})$.
}

\uIf{$|V_t|<S$}{\label{line:greedy_selection}
Note $w^{(H)}_{t,a} \le \frac{1}{\sqrt{ST}}$ for all $a\in A_H\backslash V_t$.

Fill in $V_t$ with any arms in $ A_H\backslash V_t$.

% Update $\Phi^{(H)}_{t+1}\gets \Phi^{(H)}_t\oplus (U^{(H,1)}_t\cup U^{(H,2)}_t\cup U^{(H,0)}_t)$.
}

Select the decision $V_t$ of $S$ arms.

Observe feedback $\bra{r_{t,a}: a\in V_t}.$
}
\end{algorithm}

\begin{algorithm}[!t]\caption{Base Algorithm}
\label{alg:base_alg}
\textbf{Input:} A sequence of sets of selected arms $\Phi_t=(\Phi_t(s))_{s<t}$ with $\Phi_t(s)\subseteq [K]$ and $|\Phi_t(s)|\le S$, an active set $A\subseteq [K]$. 

Set $\beta\gets \sqrt{\log(2KT)}$ and $\lambda \gets 1$. 

$A_{t}\leftarrow \lambda I + \sum_{s=1}^{t-1}\sum_{a\in \Phi_t(s)}x_{s,a}x_{s,a}^{\top}$;

$z_{t} \leftarrow \sum_{s=1}^{t-1}\sum_{a\in \Phi_t(s)}r_{s,a}x_{s,a}$;

Compute estimators $\htheta_{t} \leftarrow A_{t}^{-1}z_{t}$.

\For{$a\in A$}
{
Compute estimated reward $\widehat{r}_{t,a} \gets \htheta_{t}^{\top}x_{t,a}$. 

Compute width $w_{t,a}\gets 2(\lambda + \beta)\|x_{t,a}\|_{A_{t}^{-1}}$.
}
\end{algorithm}

% \YH{Can you bring \cref{alg:master_alg} earlier? Right now it appears one page after you first mention it.}

To develop bounds only along certain directions, our \cref{alg:master_alg} crucially relies on the following property: At time $t$, it partitions the historical data into $H=\log(\sqrt{ST})$ stages. At stage $h=1,\dots,H$, the algorithm builds the final decision $V_t$ by only looking at information that is independent of the reward observations belonging to the current $h$-th stage (while it is allowed to use rewards belonging to other stages). If $V_t$ is not fully built at the current stage, the algorithm then uses those reward observations to eliminate suboptimal arms and proceed to the next stage $h+1$. The elimination step from the previous stage $h-1$ guarantees that, even if the algorithm's move at stage $h$ is independent from the rewards at stage $h$, the suboptimality of the selected arms remains bounded.

The sole purpose of this hierarchical elimination is to guarantee that, when focusing on each stage $h\in[H]$, the reward observations are mutually independent conditioned on the contexts. Then \cref{lem:lin_width} gives a direction-specific estimation error bound. We refer to Lemma 1 of \cite{chu2011contextual} for a proof of this result. \cref{rem:dependence_issue} explains why UCB-type algorithms (such as \texttt{LinUCB} \citep{li2010contextual,abbasi2011improved}) fail this assumption.

\begin{lemma}\label{lem:lin_width}
Let $\beta=\sqrt{\log(2KT)}$ and $\lambda=1$. In \cref{alg:base_alg}, suppose $\bra{r_{s,a}: s\in[t-1], a\in \Phi_t(s)}$ are conditionally independent given contexts $\bra{x_{s,a}: s\in[t-1], a\in\Phi_t(s)}$. Then with probability at least $1-T^{-2}$, it holds that
\begin{align*}
\abs*{\htheta_{t}^{\top}x_{t,a} - \theta_{*}^{\top}x_{t,a}} \leq 2(\beta+\lambda)\|x_{t,a}\|_{A_{t}^{-1}}.
\end{align*}
for every $a\in[K]$.
\end{lemma}

\begin{remark}\label{rem:dependence_issue}
We briefly remark on why such conditional independence fails without partitioning: At time $t+1$, we would instead condition on all contexts $\bra{x_{\tau,a}: \tau\le t, a\in v_\tau}$. It is true that $\bra{r_{t,a}: a\in V_t}$ are independent from others when conditioned on $\bra{x_{t,a}: a\in V_t}$. However, the algorithm has used all previous observations $\bra{r_{\tau,a}: \tau<t, a\in v_\tau}$ to come up with the decision $V_t$. When conditioned on $\bra{x_{t,a}: a\in V_t}$ which reveals information about $V_t$, all of the previous reward observations become mutually dependent.
\end{remark}

The hierarchical elimination algorithm is described in \cref{alg:master_alg}. At each stage $h\in[H]$, given the conditional independence of the reward observations belonging to this stage, we invoke \cref{alg:base_alg} to solve the standard ridge regression and derive an estimated parameter $\htheta_t$. Thanks to the conditional independence, the estimation error $\abs{\htheta_t^\top x_{t,a} - \theta_*^\top x_{t,a}}$ is bounded by \cref{lem:lin_width}. This estimator is then used to compute reward estimators $\widehat{r}_{t,a}$ and their confidence widths $w_{t,a}$.

\paragraph{Construction of $V_t$.} 
Given the valid reward estimators and their widths, at time $t$, \cref{alg:master_alg} constructs the decision set $V_t\subseteq [K]$ as follows. At the beginning of $t$, it initializes an empty set $V_t$, then in step (1) it continuously adds \emph{underexplored} arms $U_1$ (i.e., with a large width) to $V_t$ as it goes through the hierarchical elimination with $H$ stages. % At stage $h\in[H]$, the algorithm adds any arm to $V_t$ with a large width for exploration. 
If $|V_t|<S$ and the decision is not yet filled, every remaining arm has a uniformly small width $w^{(h)}_{t,a}\le 2^{-h}$, and in step (2) we identify the \emph{confirmed} (optimal) arms $U_2$ using this width and add them to the decision $V_t$. If $V_t$ remains unfilled, we proceed to the elimination step (3) and use the uniform width $2^{-h}$ to eliminate suboptimal arms. In the end, if there is still space in $V_t$, we simply add any remaining arms; note that the remaining arms after $H$ rounds of elimination have a sufficiently small width and hence contribute to a small regret.

\subsection{Algorithmic Properties}\label{sec:lin_alg_properties}

Similar to the previous section, we establish the key properties of \cref{alg:master_alg} that justify the elimination steps and enable our regret analysis. However, there are several important distinctions from the properties in \cref{sec:graph_alg_properties} for noncontextual combinatorial bandits. 

First, due to the presence of contextual information and the dynamic nature of the regret, we no longer maintain persistent active and confirmed sets $(\Aact, \Acon)$ as in \cref{alg:comb_arm_elim_full}. Instead, the set of candidate arms $[K]$ is re-eliminated from scratch at every round $t$ after the contexts $\{x_{t,a}\}_{a\in[K]}$ are observed. This substantially complicates the algorithm and the regret analysis, since the confidence width is no longer decreasing over time for any fixed arm $a$.

Second, to apply the hierarchical elimination framework introduced in \cref{sec:hierarchical_elim_idea} and ensure the conditional independence stated in \cref{lem:lin_width}, we must construct the decision $V_t$ across $H$ stages within each round $t$. The remaining capacity in $V_t$, i.e., $S - |V_t|$, is not fixed in advance and evolves indeterministically as the stage $h \in [H]$ proceeds. As a result, every round involves $H$ elimination steps, each with a dynamically varying target size. In particular, to handle this variability, the benchmark arm in each elimination step is defined as the arm with the $(S - |V_t|)$-th largest UCB.

For the sake of clarity, let $V_t^{(h,g)}$ denote the decision set at the end of step $(g)$ at stage $h\in[H]$, for $g=1,2,3$. Namely, the initial set is $V_t^{(0,3)}=\varnothing$, and recursively $V_t^{(h,1)} = V_t^{(h-1,3)}\cup U^{(h,1)}_t$ and $V_t^{(h,2)} =V_t^{(h,3)} = V_t^{(h,1)} \cup U^{(h,2)}_t$. Let $V_{*,t}$ be the top $S$ arms at time $t$. Note that $V_{*,t}\backslash V_t^{(h,g)}$ are the left-out optimal arms by the time of step ($g$) at stage $h$, for $g=1,2,3$. To address the distinctions above, we have the following properties:

\begin{lemma}\label{lem:alg_properties_lin}
Suppose the event in \cref{lem:lin_width} holds for every $t\in[T]$ and stage $h\in[H]$. For each time $t$, the following events holds for each stage $h\in[H]$:
\begin{enumerate}[(A)]
    \item The top $S-|V_t^{(h-1,3)}|$ arms of $V_{*,t}\backslash V_t^{(h-1,3)}$ are in $A_{h}$, i.e. remain uneliminated.

    \item Confirmed arms are optimal: $U^{(h,2)}_t\subseteq V_{*,t}\backslash V_t^{(h,1)}$.

    \item Let $i_{*,t}^{(h)} = \argmin\bra{\theta_*^\top x_{t,a}: a\in A_h}$ denote the optimal active arm. For every remaining $a\in A_h$, it holds that
    \[
    \theta_*^\top x_{t,i^{(h)}_{*,t}} - \theta_*^\top x_{t,a} \le 16\cdot 2^{-h}.
    \]
\end{enumerate}
\end{lemma}

\cref{lem:alg_properties_lin} summarizes the contextual counterparts of the properties in \cref{lem:alg_properties_graph}. Claim (A) states that \cref{alg:master_alg} keeps a sufficient number of the optimal arms from $V_{*,t}$ in the active set after each elimination step (3). It guarantees that we never end up with an unfilled $|V_t|<S$ but no arm left in the active set $A_h\backslash V_t$. In addition, it guarantees that the top $S-|V_t^{(h-1,3)}|$ unselected optimal arms are not eliminated, when we proceed from stage $h-1$ to $h$.

The second claim (B) plays the same role as in \cref{lem:alg_properties_graph}, in the sense that the algorithm incurs no instantaneous regret by including the confirmed arms $U^{(h,2)}_t$. Recall that an optimal arm is confirmed if one would suffer an unbounded regret by not including it, so confirmation balances exploration and exploitation.

Finally, (C) bounds the gap between any active arm and the optimal active arm $i^{(h)}_{*,t}$ in $A_h$ at stage $h$. By claim (A), we always have $i^{(h)}_{*,t}\in V_{*,t}$ at any stage $h$ where the decision $V_t$ is not fully constructed. Crucially, this enables us to select \textit{any} active arm $a\in A_h$ in step (1) of \cref{alg:master_alg} solely by looking at its width $w_{t,a}^{(h)}$ and still manage to obtain a bounded regret.

\subsection{Regret Bounds}

We conclude this section with the near-optimal regret guarantee for \cref{alg:master_alg} and a matching lower bound.

\begin{theorem}\label{thm:lin_upper_bound}
\cref{alg:master_alg} (with \cref{alg:base_alg} as a subroutine) achieves
\[
\reg(\mathrm{Alg}~\ref{alg:master_alg}) = O\parr*{\log(ST)\sqrt{\log(2KT)}(\sqrt{dST}+dS)}. 
\]
% with a factor $\gamma=O()=\widetilde{O}(1)$.
\end{theorem}
% \YH{I removed $\gamma$ here. Check if $\gamma$ is used and properly defined in the proof.}

Therefore, \cref{alg:master_alg} improves on the existing result from $\widetilde{O}(d\sqrt{ST})$ to $\widetilde{O}(\sqrt{dST})$. The following lower bound complements our upper bound. 

\begin{theorem}\label{thm:lin_lower_bound}
Suppose $T\ge \max\bra{4dS, \frac{d^3}{S}}$ and $K\ge 2S$. For any policy $\pi$, it holds that
\[
\max_{\theta_*, \bra{x_{t,a}}}\reg(\pi) = \Omega(\sqrt{dST})
\]
where the maximum is taken over all problem instances as described in \cref{sec:lin_context_intro}.
\end{theorem}

\section{Constrained Decision}\label{sec:constrained_subset}
The combinatorial bandit literature also considers the constrained setting where the learner can only choose decision $V_t\in \Acal_0\subsetneq \binom{[K]}{S}$. The optimal regret crucially depends on the structure of the specific subset $\Acal_0$. The sub-optimality gap $\Delta_*$ in this case denotes the reward gap between the optimal decision and the second optimal decision. For example, \cite{kveton2015tight} shows that $O\parr*{\frac{KS\log(T)}{\Delta_*}}$ is optimal for a specific constrained subset $\Acal_0$. Under general graph feedback, \cite{wen2025adversarial} shows that the minimax regret is $\widetilde{\Theta}(\min\bra{S\sqrt{\alpha T}, \sqrt{KST}})$ when $\Acal_0$ is allowed to be any subset, as opposed to $\widetilde{\Theta}(S\sqrt{T}+\sqrt{\alpha ST})$ when $\Acal_0 = \binom{[K]}{S}$.

It is straightforward to extend our \cref{alg:comb_arm_elim_full} and its analysis to a general $\Acal_0$:
% ; see details in \Cref{app:constrained_subset}.
\begin{theorem}\label{thm:general_subset_upper}
There exists a policy $\pi$ such that, for any constrained decision subset $\Acal_0\subseteq \binom{[K]}{S}$, it achieves
\[
\reg(\pi) = O\parr*{\log(TK)\frac{S^2\kappa}{\Delta_*}}
\]
where $\kappa$ is a quantity the depends jointly on the feedback graph $G$ and the subset structure $\Acal_0$:
\begin{align*}
\kappa \coloneqq \max_{A\subseteq [K]}\min\{&n\ge 1: \exists V_1,\dots, V_n\in \Acal_0 \text{ such that } A\subseteq \cup_{j=1}^n\Nout(V_j)\}.\numberthis\label{eq:kappa}
\end{align*}
In particular, there is a subset $\Acal_0$ under which
\[
\min_\pi\max_\nu\reg_\nu(\pi) = \Omega\parr*{\log(T)\frac{S^2\kappa}{\Delta_*}}
\]
where the maximum is taken over all bandit environments $\nu$.
\end{theorem}
At a high level, $\kappa$ denotes the minimum number of decisions needed to observe any subgraph of $G$. \Cref{thm:general_subset_upper} recovers the result of \cite{kveton2015tight} where $\kappa=\Theta(K/S)$.

% \section*{Acknowledgement}

% Acknowledgments---Will not appear in anonymized version
% \acks{We thank a bunch of people and funding agency.}

\clearpage 

\bibliographystyle{alpha}
\bibliography{Preprint.bib}

\clearpage 
\appendix

% \crefalias{section}{appendix} % uncomment if you are using cleveref

\section{Combinatorial Bandits with Graph Feedback}

This section provides proofs for the results in \cref{sec:graph_feedback}. 

\subsection{Proof of \cref{lem:alg_properties_graph}}\label{app:alg_properties_graph_proof}

For the readers' convenience, we restate the lemma below:
\begin{lemma}[Restatement of \cref{lem:alg_properties_graph}]
Suppose the event in \cref{lem:sto_reward_concentration} holds. Then for each time $t\in[T]$, the following events hold for \cref{alg:comb_arm_elim_full} by the end of time $t$:
\begin{enumerate}[(A)]
    \item \textbf{(Optimal are not eliminated)} The optimal arms remain uneliminated, i.e. $[S]\subseteq \Aact^t\cup\Acon^t$.

    \item \textbf{(Confirmed are optimal)} $\Acon^t\subseteq [S-1]$.

    \item \textbf{(Active gaps are bounded)} Let $i_{*,t}=\min\Aact^t$. For every $a\in\Aact^t$, it holds that $\Delta_{a,i_{*,t}} \le 8w(N^t)$.
\end{enumerate}
\end{lemma}

We prove each claim separately in the remaining section. Recall that the confirmed arms $\Acon^t$ are always pulled, so their observation counts are at least $N^t$, and by \cref{lem:sto_reward_concentration}, the same confidence width $w(N^t)$ also holds with high probability for confirmed arm $i\in\Acon^t$.

\begin{lemma}
\label{lem:ith_arm_confidence_bound}
Suppose the event in Lemma~\ref{lem:sto_reward_concentration} holds. Recall that $\br_{t,(j)}$ denotes the $j$-th empirically best reward in $\Acon\cup\Aact$ at the end of time $t$, and $a_{t,(j)}$ denotes the corresponding arm. Then the followings hold by the end each time $t\in[T]$:
\begin{enumerate}
    \item[(A)] The optimal arms remain uneliminated, i.e. $[S] \subseteq \Acon^t\cup\Aact^t$.

    \item[(A')] For every index $i\in[S]$, we have $\br_{t,i} \ge \br_{t,(i)} - 2w(N^t)$.
\end{enumerate}
\end{lemma}
\begin{proof}
We prove this result via an inductive argument on time $t$. Note event (A') trivially holds at time $t=1$, as $w(N^1) \ge 1$, and (A) holds at $t=0$, namely before the algorithm kicks off.

First, we show that event (A') holds at time $t$ conditioned on event (A) at time $t-1$. Fix any $i\in[S]$ and consider the empirically top $i$ arms $\bra{a_{t,(j)}}_{j\le i}$. By the pigeonhole principle, there always exists an index $j\le i$ such that $a_{t,(j)} \ge i$. Then
\begin{align*}
\br_{t,i} &\ge \mu_{i}-w(N^t) \ge \mu_{a_{t,(j)}} - w(N^t) \ge \br_{t,(j)} - 2w(N^t) \ge \br_{t,(i)} - 2w(N^t).
\end{align*}

% If $\bra{a^t_{j-\max}}_{j<i}=\bra{1,2,\dots,i-1}$, i.e. the empirically first $i-1$ arms align with the actual $[i-1]$, then $i\le a^t_{i-\max}$. Since event (A) holds, we have $n^t_i \ge N^t$ for every $i\in[S]$. By Lemma~\ref{lem:sto_reward_concentration},
% \[
% \br^t_{i} \ge \mu_{i}-w(N^t) \ge \mu_{a^t_{i-\max}} - w(N^t) \ge \br^t_{i-\max} - 2w(N^t).
% \]
% Otherwise, there exists an index $j<i$ such that $a^t_{j-\max} \ge i$. Then we still have
% \begin{align*}
% \br^t_{i} &\ge \mu_{i}-w(N^t) \ge \mu_{a^t_{j-\max}} - w(N^t) \ge \br^t_{j-\max} - 2w(N^t) \ge \br^t_{i-\max} - 2w(N^t).
% \end{align*}

It remains to verify event (A) for any time $t$, conditioned on event (A') at $t$ and event (A) at $t-1$. If the elimination in Line \ref{eq:elim_step_active} of \cref{alg:comb_arm_elim_full} does not occur at the end of time $t$, event (A) trivially holds for $t$. Suppose instead the elimination occurs at the end of time $t$. For every $i\in[S]$, event (A') implies $\br_{t,i} \ge \br_{t,(S)} - 2w(N^t)$, so $i$ is not eliminated and event (A) holds for time $t$. This concludes the induction.
\end{proof}

\begin{lemma}\label{lem:optimal_upper_set}
Suppose the event in Lemma~\ref{lem:sto_reward_concentration} holds. At every time $t\in[T]$, it holds that $\Acon^t \subseteq [S-1]$, i.e. every confirmed arm is one of the top $S-1$ optimal arms.
\end{lemma}
\begin{proof}
It suffices to verify that when any active arm $i_0$ is added to the confirmed set $\Acon^{\tau_0}$ at time $\tau_0$, it always holds that $i_0\in[S-1]$. By the elimination step in \cref{alg:comb_arm_elim_full}, we have $\br_{\tau_0,i_0} > \br_{\tau_0,(S)} + 4w(N^{\tau_0})$. Since $\br_{\tau_0,(S)}$ is the $S$-th empirically optimal reward in $\Acon^{\tau_0}\cup\Aact^{\tau_0}$ and $[S]\subseteq \Acon^{\tau_0}\cup\Aact^{\tau_0}$ by \cref{lem:ith_arm_confidence_bound}, there is some $j\in[S]$ with $\br_{\tau_0,j} \le \br_{\tau_0,(S)}$. Then
\begin{align*}
\mu_{i_0} &\ge \br_{\tau_0,i_0} - w(N^{\tau_0}) \ge \br_{\tau_0,(S)} + 3w(N^{\tau_0}) \ge \br_{\tau_0,j} + 3w(N^{\tau_0}) \ge \mu_{j} + 2w(N^{\tau_0}) > \mu_j
\end{align*}
which implies $i_0\in[S-1]$.
\end{proof}

\begin{lemma}\label{lem:bounded_active_gaps}
Suppose the event in \cref{lem:sto_reward_concentration} holds. Let $i_{*,t}=\min \Aact^t$. For every arm $a\in\Aact^t$, we have $\Delta_{a,i_{*,t}} \le 8w(N^t)$.
\end{lemma}
\begin{proof}
\cref{lem:ith_arm_confidence_bound} and \ref{lem:optimal_upper_set} imply that $\Aact^t\cap[S]\neq \varnothing$ for all time $t$, so $i_{*,t}\in[S]$. Since $i_{*,t}\notin\Acon^t$ by definition, we have
\[
\br_{t,i_{*,t}} \le \br_{t,(S)} + 4w(N^t).
\]
Then for any active arm $a\in\Aact^t$, by the elimination criterion in \cref{alg:comb_arm_elim_full} and \cref{lem:sto_reward_concentration}, 
\[
\mu_a \ge \br_{t,a}-w(N^t) \ge \br_{t,(S)}-3w(N^t) \ge \br_{t,i_{*,t}} - 7w(N^t) \ge \mu_{i_{*,t}} - 8w(N^t).
\]
\end{proof}

\subsection{Instance-dependent Regret Upper Bound}

\begin{theorem}[Restatement of \Cref{thm:gap_dependent_reg_upper_bound}]
Fix any $\delta\in(0,1)$. With probability at least $1-\delta$, \cref{alg:comb_arm_elim_full} achieves regret
\[
\reg(\mathsf{Alg}~\ref{alg:comb_arm_elim_full}) = O\parr*{\log(2KT/\delta)\frac{\alpha\log^2 K+S}{\Delta_*}}.
\]
\end{theorem}
\begin{proof}
WLOG, suppose the high-probability event in \cref{lem:sto_reward_concentration} holds. Recall the definition $i_{*,t}=\min\Aact^t$ from \cref{lem:alg_properties_graph}. An important observation is that, for any suboptimal arm $a\in[K]\backslash[S]$, if the minimum count satisfies $N^t > 64\log(2KT/\delta)/\Delta_{a,i_{*,t}}^2$ at time $t$, then
\begin{align*}
\br_{t,(S)} - 2w(N^t) &\stepa{\ge} \br_{t,i_{*,t}} - 6w(N^t) \ge \mu_{i_{*,t}} - 7w(N^t) \\
&\stepb{>} \mu_a + w(N^t) \ge \br_{t,a}\numberthis\label{eq:gap_dependent_layer_bound}
\end{align*}
where (a) follows from the fact that $i_t\notin\Acon^t$ and the elimination criterion in Line \ref{eq:elim_step_confirmed} of \cref{alg:comb_arm_elim_full}, and (b) applies $\mu_{i_{*,t}}-\mu_a = \Delta_{a,i_{*,t}}>8w(N^t)$. Hence $a$ has already been eliminated from $\Acon^t\cup\Aact^t$.

To bound the cumulative regret, we will partition the horizon into a few sub-horizons, and over each sub-horizon we have a fixed optimal active arm $i_{*,t}$. Specifically, denote two non-repeating sequences of indices as follows: $\pi_1 = 1$ and $t_1=0$ since $1\in\Aact^0$ in the beginning; then $t_{h+1} = \min\bra{t\in[T]: i_{*,t}\notin \bra{\pi_1,...,\pi_h}}$ be the next time when the optimal $i_{*,t}$ changes in the active set, and $\pi_{h+1} = i_{*,t_{h+1}}$ be the new optimal active arm (when the previous $\pi_{h}$ is sent to the confirmed set). Since $\pi_h\in[S]$ and is increasing by definition, we end up with a sequence $(\pi_h, t_h)_{h\in[H]}$ for some $H\le S$. Namely, during the horizon $t\in[t_h, t_{h+1})$, the optimal active arm is $i_{*,t}=\pi_h$. For the ease of notation, at each of those times $t_h$, denote the minimum count as $N^{t_h} = n_h$ and define the layers 
\[
L_n = \bra{a\in[K]\backslash[S]: a \text{ is pulled as an active arm when $N^t=n$}}
\]
for $n\ge 1$. Denote $n_1=1$ and
\[
n_{H+1} = 1 + \max\bra{n\in\mathbb{N}: L_n\cap[S]\neq \varnothing}
\]
be the last layer that ever pulls a suboptimal arm (and incurs any regret). We set $n_1=1$ to keep the following computation compact.

Recall that in \cref{alg:comb_arm_elim_full}, we construct $V_t$ by repeatedly choosing the arm with the most out-neighbors among the least observed arms $\Acal_0$ and removing its neighbors. By \cref{lem:greedy_dominating_subset} and \ref{lem:alon15}, the active arms we have chosen before observing every arm in $\Acal_0$, while $N^t=n$, form a layer and can be bounded by
\begin{equation}\label{eq:layer_card_bound}
|L_n| \le 2\log K \max_{A\subseteq [K]}\delta(G|_A) + S \le 100\alpha\log^2 K + S
\end{equation}
for every layer $n\ge 0$, where $\delta(G|_A)$ is the dominating number of the subgraph $G$ restricted to the subset $A\subseteq [K]$ as defined in \eqref{eq:def_dom_number}.
Then the cumulative regret is bounded as
\begin{align*}
\reg(\mathsf{Alg}~\ref{alg:comb_arm_elim_full}) &\le |L_0| + \sum_{h=1}^{H}\sum_{n=n_h}^{n_{h+1}-1}\sum_{a\in L_n}\Delta_{a,\pi_h}\\
&\stepc{\le} |L_0| + \sum_{h=1}^{H}\sum_{n=n_h}^{n_{h+1}-1}\sum_{a\in L_n}8w(n)\\
&\stepd{\le} 100\alpha\log^2 K+S + 8(100\alpha\log^2 K+S)\sum_{h=1}^{H}\sum_{n=n_h}^{n_{h+1}-1}w(n)\\
&= 100\alpha\log^2 K+S + 8(100\alpha\log^2 K+S)\underbrace{\sum_{n=1}^{n_{H+1}-1}w(n)}_{(\spadesuit)}\numberthis\label{eq:instance_depend_reg_decompose}
\end{align*}
where (c) is by \cref{lem:alg_properties_graph}, and (d) uses \eqref{eq:layer_card_bound}. 

By \eqref{eq:gap_dependent_layer_bound}, we have $n_{H+1}-1 \le 64\log(2KT/\delta)/\Delta_*^2$. Consequently,
\begin{align*}
(\spadesuit) &\le 2\sqrt{\log(2KT/\delta)}\sqrt{n_{H+1}-1}
\le 16\log(2KT/\delta)\Delta_*^{-1}
\end{align*}
where the first inequality follows from the elementary inequality that $\sum_{k=1}^n1/\sqrt{k} \le 2\sqrt{n}$. Putting back to \eqref{eq:instance_depend_reg_decompose}, we obtain
\[
\reg_T(\mathrm{Alg~\ref{alg:comb_arm_elim_full}}) = O\parr*{\log(2KT/\delta)\frac{\alpha\log^2 K+S}{\Delta_*}}.
\]
\end{proof}

\subsection{Instance-dependent Regret Lower Bound}\label{app:instance_lower_bound}
In this section, we prove instance-dependent lower bounds for a given gap $\Delta_*$ in \Cref{thm:gap_dependent_reg_lower_bound}.
\begin{theorem}
Suppose $\alpha\ge 2S$ and fix a policy $\pi$ that satisfies $\reg(\pi)\le C T^p$ for some constant $C>0$ and $p\in[0,1)$ under any bandit environment. If $\frac{1}{4}\ge\Delta_*>T^{-(1-p)}$, then it holds that
\[
\max_\nu \reg_\nu(\pi) = \Omega\parr*{\log(T^{1-p}\Delta_*)\min\bra*{\frac{\alpha}{\Delta_*},\frac{1}{\Delta_*^2}} - \frac{\log(C)}{\Delta_*}}
\]
where the maximum is taken over all bandit environment $\nu$ and $\reg_\nu(\pi)$ is the expected regret of $\pi$ under the environment $\nu$.
\end{theorem}
\begin{proof}
Let $I=\bra{a_1,\dots, a_\alpha}$ be a maximum independent subset in $G$. We will fix the first $a_1,\dots, a_{S-1}$ arms to be optimal, and construct $\alpha-S+1$ different environments. Specifically, consider an index $u\in\bra{S,S+1,\dots, \alpha}$. To define an environment based on index $u$, let the product reward distribution $P^u = \prod_{a\in[K]}\mathrm{Bern}(\mu_a)$ with
\[
\mu_a = \begin{cases}
    1 & \text{if $a=a_i$ for some $i\in[S-1]$}\\
    \frac{1}{4}+\Delta_* & \text{if $a=a_{S}$}\\
    \frac{1}{4}+2\Delta_*\indic[u>S] & \text{if $u>S$ and $a=a_{u}$}\\
    \frac{1}{4} & \text{else if $a\in I$}\\
    0 & \text{otherwise}
\end{cases}.
\]

Fix any policy $\pi$. Let $\E_u$ (resp. $\mathbb{P}_u$) be the expectation (resp. probability) taken under environment $u$. Suppose $\reg_u(\pi)\le C T^p$ for some constants $C>0$ and $p\in[0,1)$, where $\reg_u(\pi)$ denotes the expected regret of $\pi$ under environment $u$. Denote $N_{i}(T) = \sum_{t=1}^T\indic[i\in V_t]$ where $V_t$ is the decision selected by $\pi$ at time $t$. Denote $N_0 = \sum_{t=1}^T\indic[V_t\cap ([K]\backslash I) \neq \varnothing]$ the number of times an arm outside $I$ is selected. Then for any $i=S+1,\dots, \alpha$,
\begin{align*}
\mathbb{P}_S\parr*{N_i \ge \frac{T}{2}} + \mathbb{P}_i\parr*{N_i < \frac{T}{2}} &\ge \frac{1}{2}\exp\parr*{-\mathsf{KL}(\mathbb{P}_S^{\otimes T}\|\mathbb{P}_i^{\otimes T})}\\
&\stepa{\ge} \frac{1}{2}\exp\parr*{-(\E_S[N_i]+\E_S[N_0])\frac{16\Delta_*^2}{3}}
\end{align*}
where (a) uses the inequality $\mathsf{KL}\parr*{\mathrm{Bern}(p)\|\mathrm{Bern}(q)} \le \frac{(p-q)^2}{q(1-q)}$ and $\Delta_*\in(0,\frac{1}{4}]$. By construction, for $i=S+1,\dots,\alpha$,
\[
\reg_S(\pi) + \reg_i(\pi) \ge \frac{\Delta_* T}{4}\parr*{\mathbb{P}_S\parr*{N_i \ge \frac{T}{2}} + \mathbb{P}_i\parr*{N_i < \frac{T}{2}}} \ge \frac{\Delta_* T}{8}\exp\parr*{-(\E_S[N_i]+\E_S[N_0])\frac{16\Delta_*^2}{3}}
\]
which implies
\[
\E_S[N_i]+\E_S[N_0] \ge \frac{3}{16\Delta_*^2}\log\parr*{\frac{\Delta_* T}{8(\reg_S(\pi)+\reg_i(\pi))}} \ge \frac{3}{16\Delta_*^2}\log\parr*{\frac{\Delta_* T^{1-p}}{16C}}
\]
Then we have
\begin{align*}
\reg_S(\pi) &\ge \Delta_*\sum_{i=S+1}^\alpha\E_S[N_i] + \frac{1}{4}\E_S[N_0]\\
&\ge \Delta_*\sum_{i=S+1}^\alpha\parr*{\E_S[N_i] + \E_S[N_0]} + \parr*{\frac{1}{4}-\Delta_*(\alpha-S)}\E_S[N_0]\\
&\ge \frac{3}{16\Delta_*}\sum_{i=S+1}^\alpha\log\parr*{\frac{\Delta_* T^{1-p}}{16C}} + \parr*{\frac{1}{4}-\Delta_*(\alpha-S)}\E_S[N_0]\\
&\ge \frac{3\alpha}{32\Delta_*}\log\parr*{\frac{\Delta_* T^{1-p}}{16C}}+ \parr*{\frac{1}{4}-\Delta_*\alpha}\E_S[N_0]\numberthis\label{eq:reg_lower_eq2}
\end{align*}
where the last inequality uses $\alpha\ge 2S$. On the other hand, we also have
\begin{equation}\label{eq:reg_lower_eq1}
\reg_S(\pi) \ge \frac{1}{4}\E_S[N_0].
\end{equation}
When $\Delta_* \le \frac{1}{4\alpha}$, \eqref{eq:reg_lower_eq2} directly gives the desired bound since $1-\Delta_*\alpha\ge 0$. When $\Delta_*>\frac{1}{4\alpha}$, we have
\begin{align*}
\reg_S(\pi) &\overset{\eqref{eq:reg_lower_eq2}}{\ge} \frac{3\alpha}{32\Delta_*}\log\parr*{\frac{\Delta_* T^{1-p}}{16C}} - \Delta_*\alpha\E_S[N_0]\\
&\overset{\eqref{eq:reg_lower_eq1}}{\ge} \frac{3\alpha}{32\Delta_*}\log\parr*{\frac{\Delta_* T^{1-p}}{16C}} - \frac{\Delta_*\alpha}{4}\reg_S(\pi).
\end{align*}
Then 
\[
\frac{17\Delta_*\alpha}{4}\reg_S(\pi) \ge \frac{4+\Delta_*\alpha}{4}\reg_S(\pi) \ge  \frac{3\alpha}{32\Delta_*}\log\parr*{\frac{\Delta_* T^{1-p}}{16C}}
\]
which concludes the proof.
\end{proof}

\begin{theorem}
Suppose $K\ge 2S$ and $\Delta_*\le \frac{1}{2}$. For any policy $\pi$, it holds that
\[
\max_\nu \reg_\nu(\pi) = \Omega\parr*{S\min\bra{\Delta_* T,\Delta_*^{-1}}}
\]
where the maximum is taken over all bandit environment $\nu$ and $\reg_\nu(\pi)$ is the expected regret of $\pi$ under the environment $\nu$.
\end{theorem}
\begin{proof}
Note it suffices to prove the lower bound for the full-information feedback, since a policy can recover its performance under other feedback by simply discarding observations obtained under full information. Since $K\ge 2S$, we can fix any two disjoint subsets $V_0, V_1\subseteq [K]$ with $|V_0|=|V_1|=S$. Write $V_0=\bra{a_{0,1},\dots,a_{0,S}}$ and $V_1=\bra{a_{1,1},\dots,a_{1,S}}$ respectively. We will construct the bandit environments as follows. For index $u\in\bra{0,1}^S$, define the product reward distribution $P^u = \prod_{a\in[K]}\mathrm{Bern}(\mu_a)$ with
\[
\mu_a = \begin{cases}
    \frac{1}{4}+\Delta_* & \text{if $a=a_{u_j,j}$ for some $j\in[S]$}\\
    \frac{1}{4} & \text{if $a=a_{1-u_j,j}$ for some $j\in[S]$}\\
    0 & \text{otherwise}
\end{cases}.
\]
Let $\mathbb{P}_u$ be the probability taken under environment $u$. Under full information, WLOG, we assume the policy only pulls arms in $V_0\cup V_1$ that have positive rewards. Consider sampling $u\in\bra{0,1}^S$ uniformly. We lower bound the minimax regret by the Bayes regret:
\begin{align*}
\max_u\reg_u(\pi) &\ge 2^{-S}\sum_{u\in\bra{0,1}^S}\E_u\parq*{\Delta_*\sum_{t=1}^T\sum_{j=1}^S\indic[a_{1-u_{j,j}}\in V_t]}\numberthis\label{eq:instance_bayes_reg}
% &\ge \frac{S\Delta_*}{8}\sum_{t=1}^T\parr*{\mathbb{P}^{\otimes(t-1)}_0\parr*{|V_0\cap V_t|< \frac{S}{2}} + \mathbb{P}^{\otimes(t-1)}_1\parr*{|V_0\cap V_t|\ge  \frac{S}{2}}}\\
% &\ge \frac{S\Delta_*}{16}\sum_{t=1}^T\exp\parr*{-\mathsf{KL}(\mathbb{P}_0^{\otimes (t-1)}\|\mathbb{P}_1^{\otimes (t-1)})}\\
% &\stepa{\ge} \frac{S\Delta_*}{16}\sum_{t=1}^T\exp\parr*{-\frac{16(t-1)\Delta_*^2}{3}}\\
% &\ge \frac{S\Delta_*}{8}\sum_{t=1}^{\min\bra{T,\Delta_*^{-2}}}\exp\parr*{-\frac{16}{3}}\\
% &\ge cS\Delta_*\min\bra{T,\Delta_*^{-2}}
\end{align*}
% for an absolute constant $c=\exp(-16/3)/16$, where (a) applies the inequality $\mathsf{KL}\parr*{\mathrm{Bern}(p)\|\mathrm{Bern}(q)} \le \frac{(p-q)^2}{q(1-q)}$ and $\frac{\Delta_*}{2}\in[0,\frac{1}{2}]$.
Now we consider a uniform random variable $J\sim\mathrm{Unif}([S])$. By \eqref{eq:instance_bayes_reg},
\begin{align*}
\max_u\reg_u(\pi) &\ge 2^{-S}\sum_{u\in\bra{0,1}^S}\E_u\parq*{S\Delta_*\sum_{t=1}^T\E_{J\sim\mathrm{Unif}([S])}\parq*{\indic[a_{1-u_{J},J}\in V_t]}}\\
&= S\Delta_*\E_{J\sim\mathrm{Unif}([S])}\parq*{2^{-S}\sum_{u\in\bra{0,1}^S}\E_u\parq*{\sum_{t=1}^T\indic[a_{1-u_{J},J}\in V_t]}}. \numberthis\label{eq:instance_bayes_reg_2}
\end{align*}
For each $J\in[S]$ and $u$, denote $u_{-J}$ as $u$ except the $J$-th entry, $u^{(0)}_{-J}$ as setting $u_J$ to $0$, and $u^{(1)}_{-J}$ as setting $u_J$ to $1$. Let $\Fcal_t$ denote the history up to the beginning of time $t$. Then we have
\begin{align*}
\eqref{eq:instance_bayes_reg_2} &\stepa{\ge} S\Delta_*\E_{J\sim\mathrm{Unif}([S])}\parq*{2^{-S}\sum_{u_{-J}\in\bra{0,1}^{S-1}}\E_{u_{-J}^{(0)}}\parq*{\sum_{t=1}^T\indic[a_{1,J}\in V_t]}+\E_{u_{-J}^{(1)}}\parq*{\sum_{t=1}^T\indic[a_{1,J}\notin V_t]}}\\
&= S\Delta_*\E_{J\sim\mathrm{Unif}([S])}\parq*{2^{-S}\sum_{u_{-J}\in\bra{0,1}^{S-1}}\sum_{t=1}^T\mathbb{P}_{u_{-J}^{(0)}}(a_{1,J}\in V_t|\Fcal_t)+\sum_{t=1}^T\mathbb{P}_{u_{-J}^{(1)}}\parr*{a_{1,J}\notin V_t|\Fcal_t}}\\
&\stepb{\ge} S\Delta_*\E_{J\sim\mathrm{Unif}([S])}\parq*{2^{-S-1}\sum_{u_{-J}\in\bra{0,1}^{S-1}}\sum_{t=1}^T\exp\parr*{-\mathsf{KL}\parr*{\mathbb{P}_{u^{(0)}_{-J}}^{\otimes(t-1)} \big\| \mathbb{P}_{u^{(1)}_{-J}}^{\otimes(t-1)}}}}\\
&\stepc{\ge} S\Delta_*\E_{J\sim\mathrm{Unif}([S])}\parq*{2^{-S-1}\sum_{u_{-J}\in\bra{0,1}^{S-1}}\sum_{t=1}^T\exp\parr*{-\frac{16(t-1)\Delta_*^2}{3}}}\\
&= \frac{S\Delta_*}{4}\sum_{t=1}^T\exp\parr*{-\frac{16(t-1)\Delta_*^2}{3}}\\
&\ge \frac{S\Delta_*}{4}\sum_{t=1}^{\min\bra{T,\Delta_*^{-2}}}\exp\parr*{-\frac{16}{3}} = \Omega\parr*{S\Delta_*\min\bra{T,\Delta_*^{-2}}}.
\end{align*}
(a) uses a key observation that, for fixed $J\in[S]$, when $a_{J,1}$ is optimal (under environment $u_{-J}^{(1)}$) but $a_{J,1}\notin V_t$, $V_t$ must have included a suboptimal arm (not necessarily $a_{J,0}$) and thereby incur regret $\Delta_*$. (b) applies the Bretagnolle–Huber inequality up to time $t$. (c) follows from the chain rule of KL, the inequality $\mathsf{KL}\parr*{\mathrm{Bern}(p)\|\mathrm{Bern}(q)} \le \frac{(p-q)^2}{q(1-q)}$, and $\Delta_*\in[0,\frac{1}{2}]$.
\end{proof}

\subsection{Minimax Regret Upper Bound}

\begin{theorem}[Restatement of \Cref{thm:minimax_reg_upper_bound}]
Fix any $\delta\in(0,1)$. With probability at least $1-\delta$, 
\[
\reg(\mathrm{Alg}~\ref{alg:comb_arm_elim_full}) = O\parr*{\log^2 K\sqrt{\log(2TK/\delta)}\parr*{\sqrt{\alpha ST} + S\sqrt{T}}}.
\]
\end{theorem}
\begin{proof}
Consider the instantaneous regret $\sum_{i=1}^S\mu_i - \sum_{a\in V_t}\mu_a$ at time $t$. By \eqref{eq:inst_reg_cancellation}, the instantaneous regret comes from arms in the active set and can be bounded as:
\begin{equation}\label{eq:inst_reg_bound_by_Acon}
\sum_{i=1}^S\mu_i - \sum_{a\in V_t}\mu_a \le \sum_{a\in V_t\cap\Aact^t}(\mu_{i_t} - \mu_a) 
\le 8|V_t\cap\Aact^t|\cdot w(N^t)
= 8\parr*{S-|\Acon^t|}w(N^t)
\end{equation}
where the second inequality follows from \cref{lem:alg_properties_graph}.
Since $|\Acon^t|$ is increasing in $t$ and bounded by $S-1$, we partition the horizon $[T]$ by 
\[
1=t_0\le t_1 \le \cdots \le t_S=T+1.
\]
For each $s=0,1,\dots,S-1$, during the sub-horizon $t\in[t_{s}, t_{s+1})$, it holds that $|\Acon^t| = s$. Denote $T(n)$ to be the number of times when the minimum count $N^t = n$, and recall that $N^{t_s}$ is the minimum count at the time $\Acon^t$ is updated. Then by \eqref{eq:inst_reg_bound_by_Acon}, the regret is bounded by
\begin{equation}\label{eq:minimax_regret_decompose1}
\reg(\mathrm{Alg}~\ref{alg:comb_arm_elim_full}) \le 8\sum_{s=0}^{S-1}\sum_{n=N^{t_s}}^{N^{t_{s+1}}-1}T(n)(S-s)\min\bra{1,w(n)}.
\end{equation}
Similar to \eqref{eq:layer_card_bound}, we now bound the number of times when $N^t=n$, i.e. $T(n)$, by bounding the number of active arms pulled during these times. Recall that in \cref{alg:comb_arm_elim_full}, we construct $V_t$ by repeatedly choosing the arm with the most out-neighbors among the least observed arms $\Acal_0$ and removing its neighbors. By \cref{lem:greedy_dominating_subset} and \ref{lem:alon15}, the active arms we have chosen before observing every arm in $\Acal_0$, while $N^t=n$, can be bounded by
\begin{equation}\label{eq:layer_card_bound2}
T(n)(S-|\Acon^t|) \le \max\bra*{2\log K \max_{A\subseteq [K]}\delta(G|_A) , S-|\Acon^t|} \le \max\bra*{100\alpha\log^2 K , S-|\Acon^t|}
\end{equation}
where $\delta(G|_A)$ is the dominating number of the subgraph $G$ restricted to $A$, as defined in \eqref{eq:def_dom_number}. The maximum comes from the fact that we experience at least one time step before updating the minimum count $N^t$ (i.e. $T(n)\ge 1$).
Then during each sub-horizon $t\in[t_s,t_{s+1})$, as $|\Acon^t|=s$, we have 
\begin{equation}\label{eq:Tn_bound}
T(n) \le \overline{T}_s \coloneqq \ceil{\frac{100\alpha\log^2 K}{S-s}}
\end{equation}
for $n=N^{t_s},\dots,N^{t_{s+1}}-1$. On the other hand, by definition, we also have the following constraints. For each $s=0,1,\dots,S-1$,
\begin{equation}\label{eq:Nt_constraint}
\sum_{n=N^{t_s}}^{N^{t_{s+1}}-1}T(n) = t_{s+1}-t_s.
\end{equation}

To bound the regret decomposition in \eqref{eq:minimax_regret_decompose1}, we consider the maximization of the right-hand side over the possible values of $0=N^{t_0}\le N^{t_1}\le\cdots \le N^{T_S}$ subject to the constraints \eqref{eq:Tn_bound} and \eqref{eq:Nt_constraint}, for any given partition $\bra{t_s: s=0,1,\dots,S}$. The maximum is achieved (or upper bounded, if the sub-horizon $t_{s+1}-t_s$ is not divisible by $\overline{T}_s$) by setting
\begin{equation}\label{eq:regret_rhs_maximizer}
N^{t_{s+1}}_* = N^{t_s}_* + \ceil{\frac{t_{s+1}-t_s}{\overline{T}_s}},\quad\quad\text{and }T(n)=\overline{T}_s
\end{equation}
where $\overline{T}_s$ is the upper bound described in \eqref{eq:Tn_bound}. To see this, if $T(n)<\overline{T}_s$ for some ``earlier" $n$ at the $s$-th sub-horizon, the right-hand side of \eqref{eq:minimax_regret_decompose1} can be made larger by increasing $T(n)$ to $\overline{T}_s$ and decreasing $T(N^{t_{s+1}-1})$ to satisfy \eqref{eq:Nt_constraint}, because the multiplicative factor $w(n)$ is decreasing in $n$. Namely, the target quantity in \eqref{eq:minimax_regret_decompose1} is larger if we fill in $T(n)$ for small $n$ first.

Consequently, the regret in \eqref{eq:minimax_regret_decompose1} is bounded as
\begin{align*}
\reg(\mathrm{Alg}~\ref{alg:comb_arm_elim_full})
&\le 8\sum_{s=0}^{S-1}\sum_{n=N_*^{t_s}}^{N_*^{t_{s+1}}-1}\overline{T}_s(S-s)\min\bra{1,w(n)}\\
&\stepa{\le} 8S\overline{T}_0 + L\sum_{s=0}^{S-1}\overline{T}_s(S-s)\parr*{\sqrt{N_*^{t_{s+1}}-1} - \sqrt{N_*^{t_s}-1}}\numberthis\label{eq:regret_layered_decompose}
\end{align*}
where (a) uses $\min\bra{1,w(0)}\le 1$, the elementary inequality $\sum_{t=i}^j1/\sqrt{t} \le 2\parr*{\sqrt{j}-\sqrt{i-1}}$, and for the width factor $L=16\sqrt{\log(2KT/\delta)}$. Now we distinguish into two cases.

First suppose $\alpha \ge S$, so 
\begin{equation*}
C_0\frac{\alpha\log^2 K}{S-s} \le \overline{T}_s \le C_1\frac{\alpha\log^2 K}{S-s}
\end{equation*}
for some absolute constants $C_0, C_1>0$. Then following \eqref{eq:regret_layered_decompose}, we have $C_0\alpha\log^2 K \le \overline{T}_s(S-s) \le C_1\alpha\log^2 K$ for every $s=0,1,...,S-1$ and
\begin{align*}
\reg(\mathrm{Alg}~\ref{alg:comb_arm_elim_full})
&\le 8C_1\alpha\log^2 K + C_1L\alpha\log^2 K \sqrt{N_*^{t_S}-1}\\
&\stepb{\le} 8C_1\alpha\log^2 K  + C_1L\alpha\log^2 K \sqrt{\sum_{s=0}^{S-1}\frac{S-s}{C_0\alpha\log^2 K}(t_{s+1}-t_s)}\\
&\le 8C_1\alpha\log^2 K + \frac{C_1}{\sqrt{C_0}}L\log K\sqrt{\alpha S(T+1)}
\end{align*}
where (b) plugs in the definition of $N_*^{t_S}$ in \eqref{eq:regret_rhs_maximizer}, and the last line simply bounds $S-s\le S$ and $t_S=T+1$. 

When $\alpha<S$, we have $\overline{T}_s(S-s) \le C_2S\log^2 K$ for an absolute constant $C_2>0$ for every $s=0,1,...,S-1$ and also $\overline{T}_s \ge 1$.
Then \eqref{eq:regret_layered_decompose} gives
\begin{align*}
\reg(\mathrm{Alg}~\ref{alg:comb_arm_elim_full})
&\stepc{\le} 8(100\alpha\log^2 K+S) + C_2LS\log^2K\sqrt{N_*^{t_S}-1}\\
&= 8(100\alpha\log^2 K+S) + C_2LS\log^2 K\sqrt{\sum_{s=0}^{S-1}\ceil{\frac{t_{s+1}-t_s}{\overline{T}_s}}}\\
&\le 8(100\alpha\log^2 K+S) + C_2LS\log^2 K\sqrt{\sum_{s=0}^{S-1}(t_{s+1}-t_s)}\\
&= 8(100\alpha\log^2 K+S) + C_2L S\sqrt{T}\log^2 K
\end{align*}
where (c) substitutes in $\overline{T}_s(S-s) \le C_2\log^2 KS$ and applies the telescoping sum.
\end{proof}

\subsection{Suboptimality Results for UCB}\label{app:suboptimal_ucb}

\begin{theorem}[Restatement of \Cref{thm:ucb_fail}]
Fix any $(S,\alpha,K,T)$ with $S\alpha\le K$ and $\alpha >1$. There is a problem instance under which 
\[
\reg(\mathrm{Alg}~\ref{alg:ucb}) = \Omega(LS\sqrt{\alpha T})
\]
where $L$ is the width parameter used in \cref{alg:ucb}.
\end{theorem}
\begin{proof}
Partition $[K]$ into subsets $V_1,V_2,\dots,\overline{V}_{\alpha}$, with $|V_1|=|V_2|=\cdots=|V_{\alpha-1}|=S$ and $|\overline{V}_\alpha|=K-S(\alpha-1)\ge S$. The feedback graph $G=([K],E)$ is defined by having each $V_k$ and $\overline{V}_{\alpha}$ be a clique respectively, for $k\in[\alpha-1]$. To analyze the trajectory of the UCB algorithm, we consider a \textit{deterministic} reward for each arm. Specifically, fix any subset $V_\alpha\subseteq \overline{V}_\alpha$ with size $|V_\alpha|=S$. For each $a\in V_k$ for $k=1,2,\dots,\alpha$, we define a shared, deterministic reward
\begin{equation}\label{eq:ucb_lower_bound_reward}
r_{t,a}\equiv \frac{1}{2}+\indic[i=1]\Delta-(k-1)\varepsilon
\end{equation}
for some small $\Delta,\varepsilon>0$ to be specified later, and let $r_{t,b} \equiv 0$ for any $b\in\overline{V}_{\alpha}\backslash V_\alpha$. These ``extra" arms $\overline{V}_{\alpha}\backslash V_\alpha$ are clearly suboptimal and, as we will argue next, are never pulled for more than once. It is clear that the optimal combinatorial decision is $V_1$, and 
\begin{equation}\label{eq:ucb_lower_bound_gap}
\sum_{i\in V_1}\mu_i - \sum_{a\in V_k}\mu_{a} = S\parr*{\Delta + (k-1)\varepsilon} \ge S\Delta
\end{equation}
for any other $k=2,3,\dots,\alpha$. 

We have a few key observations. First, under any tie-breaking rule and this constructed feedback graph $G$, \cref{alg:ucb} initializes every arm to exactly one observation, i.e. there is a time $t_0 \le \alpha$ such that $n_{t_0,a}=1$ for every $a\in[K]$. To see why, note that given the graph structure, an arm $a\in V_k$ has $n_{t,a}=0$ if and only if $n_{t,b}=0$ for all $b\in V_k$, at any time $t$. Meanwhile, there are at least $S$ arms in each $V_k$, so the algorithm never chooses any arm $c$ with $n_{t,c}>0$ whenever there is an arm $a\in[K]$ with $n_{t,a}=0$. These imply the existence of the time $t_0$, at which every arm is observed exactly once. The existence of such $t_0$, together with the deterministic rewards, allows us to precisely describe the algorithm's trajectory thereafter.

Due to the feedback structure of $G$, if two arms $a,b\in V_k$ or $\overline{V}_\alpha$ are in the same clique, then $n_{t,a}=n_{t,b}$ for all $t$. Thus we are guaranteed the following:
\begin{itemize}
    \item The algorithm never pulls any arm from $\overline{V}_\alpha\backslash V_\alpha$, since $V_\alpha\subseteq \overline{V}_\alpha$ are the $S$ arms with (always) higher UCB (i.e. $V_\alpha$ share the same observation counts but have higher observed empirical rewards than $\overline{V}_\alpha\backslash V_\alpha$).

    \item The algorithm will only pull one of the cliques $\bra{V_1, V_2,\dots, V_\alpha}$ at each time, due to the deterministic choice of rewards in \eqref{eq:ucb_lower_bound_reward} and if the reward is distinct for each clique, i.e. $r_{t,a}>r_{t,b}$ if $a\in V_i$ and $b\in V_j$ for $i<j$.
    % $\varepsilon\notin\bra*{\frac{L}{\sqrt{k}}:k\in[T]}$.
\end{itemize}

Before proceeding to the exact proof, we give an intuition of why UCB fails under the constructed problem instance above. As demonstrated above, the selection rule of UCB guarantees that only one of the $\alpha$ cliques will be selected at each time $t$, due to the lack of explicit exploration. The number of observations collected is always $S$ at each time. In contrast, our proposed elimination algorithm (\cref{alg:comb_arm_elim_full}) explicitly leverages the fact that the sub-optimality gap of every clique $V_k$ is of the same order, so it will be forced to distribute the $S$ selection budget over different cliques, paying an instantaneous regret of the same order as selecting only one of the cliques. This is beneficial under graph feedback, as now the number of observations collected is $\min\bra{S^2,K}$ i.e. every arm is observed. The speed of information collection is roughly $S$ times faster, shaving the regret by a factor of $\sqrt{S}$. At a high level, this example highlights that the implicit exploration in UCB is insufficient to optimally exploit general graph feedback.

Now we rigorously present the lower bound proof. The exploration-exploitation trade-off is captured by the fact that UCB needs to pull other cliques to realize that they are suboptimal, which incurs sub-optimal regret crucially because UCB always focuses on only one of the cliques. Let $t_1$ denote the last time at which $V_1$ is pulled, and $n_{t_1,1}\le t_1$ denote the total number of times $V_1$ is pulled. Due to the reward gap described in \eqref{eq:ucb_lower_bound_gap}, we have
\begin{equation}\label{eq:ucb_reg_lower_1}
\reg(\mathrm{Alg}~\ref{alg:ucb}) \ge (T-n_{t_1,1})S\Delta.
\end{equation}
In addition, since $V_1$ has the largest UCB at time $t_1$, for every $k=2,3,\dots,\alpha$, the other cliques have a smaller UCB:
\begin{align*}
\frac{1}{2} + \Delta + \frac{L}{\sqrt{n_{t_1,1}}} &\ge \frac{1}{2} - (k-1)\varepsilon + \frac{L}{\sqrt{n_{t_1,k}}} \\
&\ge \frac{1}{2} - \alpha\varepsilon + \frac{L}{\sqrt{n_{t_1,k}}}\numberthis\label{eq:last_ucb_choice}
\end{align*}
where $n_{t_1,k}$ denotes the observation count of any arm of the clique $V_k$ at time $t_1$, since they are the same. Recall that the observation count of a clique $V_k$ equals the number of times it gets pulled.

Under our graph structure, $T-n_{t_1,1} = \sum_{k=2}^\alpha n_{t_1,k}$. So \eqref{eq:ucb_reg_lower_1} translates to
\begin{equation}\label{eq:ucb_reg_lower_2}
\reg(\mathrm{Alg}~\ref{alg:ucb}) \ge \sum_{k=2}^\alpha S\Delta n_{t_1,k}.
\end{equation}
For the sake of simplicity, we write $T-n_{t_1,1} = cT$ for some $c\in[0,1]$. Choose $\Delta = L\sqrt{\alpha/(4T)}$ and $\varepsilon = L\sqrt{1/(4\alpha T)}$. Then we have $\Delta + \alpha\varepsilon = L\sqrt{\alpha/T}$. By \eqref{eq:last_ucb_choice} and some algebra,
\begin{align*}
L\parr*{\sqrt{\frac{\alpha}{T}} + \sqrt{\frac{1}{(1-c)T}}} \ge L\sqrt{\frac{1}{n_{t_1,k}}}
\end{align*}
for every $k=2,\dots,\alpha$, which implies
\begin{equation*}
 n_{t_1,k} \ge \frac{T}{2(\alpha + \frac{1}{1-c})}.
\end{equation*}
Then \eqref{eq:ucb_reg_lower_2} becomes
\begin{align*}
\reg(\mathrm{Alg}~\ref{alg:ucb}) &\ge LS\sqrt{\alpha T}\frac{\alpha-1}{4(\alpha+\frac{1}{1-c})}\ge LS\sqrt{\alpha T}\frac{\alpha}{8(\alpha+\frac{1}{1-c})}.\numberthis\label{eq:ucb_reg_lower_3}
\end{align*}
Meanwhile, \eqref{eq:ucb_reg_lower_1} gives a trade-off lower bound
\begin{equation}\label{eq:ucb_reg_lower_4}
\reg(\mathrm{Alg}~\ref{alg:ucb}) \ge \frac{cL}{2}S\sqrt{\alpha T}.
\end{equation}
When $c \in [0,\frac{1}{2}]$, \eqref{eq:ucb_reg_lower_3} gives
\begin{align*}
\reg(\mathrm{Alg}~\ref{alg:ucb}) &\ge LS\sqrt{\alpha T}\frac{\alpha}{8(\alpha + 2)}\ge \frac{L}{16}S\sqrt{\alpha T}\frac{\alpha}{\alpha+1} \ge \frac{L}{32}S\sqrt{\alpha T}
\end{align*}
and when $c\in[\frac{1}{2},1]$, \eqref{eq:ucb_reg_lower_4} gives
\[
\reg(\mathrm{Alg}~\ref{alg:ucb}) \ge \frac{L}{4}S\sqrt{\alpha T}.
\]
Since they hold simultaneously, the regret $\reg(\mathrm{Alg}~\ref{alg:ucb})$ is lower bounded by the maximum of the two, which concludes the proof.
\end{proof}

\subsection{UCB Regret Upper Bound under Full Information}\label{app:ucb_full_info}
\begin{theorem}
Suppose we have full-information feedback. Let $L=\sqrt{\log(2KT/\delta)}$. With probability at least $1-\delta$,
\[
\reg(\mathrm{Alg}~\ref{alg:ucb}) = O(S\sqrt{\log(2KT/\delta)T}).
\]
\end{theorem}
\begin{proof}
WLOG, suppose the event in the concentration result \cref{lem:sto_reward_concentration} holds. The number of observations for every arm $a\in[K]$ at time $t$ is $t-1$ under full-information. Fix any $i\in[S]$. Consider the $i$-th arm $a_{t,(i)}$ with the largest UCB (if there is a tie, consider the selected ones with arbitrary ordering) and compare with the arm $i$. Suppose $a_{t,(i)}> i$. By the pigeonhole principle, there exists an arm $j\le i$ such that $\br_{t,(i)} \ge \br_{t,j}$. Then by \cref{lem:sto_reward_concentration},
\begin{align*}
\mu_{a_{t,(i)}} \ge \br_{t,(i)} - \frac{L}{\sqrt{t-1}} \ge \br_{t,j} - \frac{L}{\sqrt{t-1}} \ge \mu_j - 2\frac{L}{\sqrt{t-1}} \ge \mu_i - 2\frac{L}{\sqrt{t-1}}.
\end{align*}
The instantaneous regret at time $t$ is thus bounded by
\[
\sum_{i=1}^S\mu_i - \sum_{a\in V_t}\mu_a \le \sum_{i=1}^S\indic[a_{t,(i)}> i]\parr*{\mu_i - \mu_{a_{t,(i)}}} \le 2\sum_{i=1}^S\indic[a_{t,(i)}> i]\frac{L}{\sqrt{t-1}} \le 2S\frac{L}{\sqrt{t-1}}.
\]
Then the cumulative regret is easily bounded as
\begin{align*}
\reg(\mathrm{Alg}~\ref{alg:ucb}) \le 2S\sum_{t=1}^T \min\bra*{1,\frac{L}{\sqrt{t-1}}} \le 4S\sqrt{\log(2KT/\delta)T}+2S.
\end{align*}
\end{proof}

\section{Combinatorial Bandits with Linear Contexts}
This section provides proofs for the results in \cref{sec:lin_context}.

\subsection{Proof of \cref{lem:alg_properties_lin}}

\begin{lemma}[Restatement of \cref{lem:alg_properties_lin}]
Suppose the event in \cref{lem:lin_width} holds for every $t\in[T]$ and stage $h\in[H]$. For each time $t$, the following events holds in \cref{alg:master_alg} for each stage $h\in[H]$:
\begin{enumerate}[(A)]
    \item The top $S-|V_t^{(h-1,3)}|$ arms of $V_{*,t}\backslash V_t^{(h-1,3)}$ are in $A_{h}$, i.e. remain uneliminated.

    \item The confirmed arms are optimal: $U^{(h,2)}_t\subseteq V_{*,t}\backslash V_t^{(h,1)}$.

    \item Let $i_{*,t}^{(h)} = \argmin\bra{\theta_*^\top x_{t,a}: a\in A_h}$ denote the optimal active arm. For every $a\in A_h$, it holds that
    \[
    \theta_*^\top x_{t,i^{(h)}_{*,t}} - \theta_*^\top x_{t,a} \le 16\cdot 2^{-h}.
    \]
\end{enumerate}
\end{lemma}
\begin{proof}
We will prove the result via an inductive argument. At $h=1$, the claims (A) and (C) trivially hold. 

We first show that claim (B) holds for any stage $h\in[H]$ conditioned on (A). By the time of step (2), the width of every remaining active arm satisfies $w^{(h)}_{t,a}\le 2^{-h}$ for $a\in A_h\backslash V^{(h,1)}_t$. By claim $(A)$, the top $S-|V_t^{(h,1)}|$ arms of $V_{*,t}\backslash V_t^{(h,1)}$ remain in the active set $A_h\backslash V^{(h,1)}_t=A_h\backslash U_t^{(h,1)}$, since the number of selected optimal arms in step (1) of \cref{alg:master_alg} is no more than the number of selected active arms $U_t^{(h,1)}$; denote the set of these arms as $S^{(h)}_t$ for simplicity. Intuitively, this subset of optimal arms serves as the benchmark that the arms added to $V_t$ later will be compared against. Under the concentration event in \cref{lem:lin_width}, if any remaining active arm $i_0\in A_h\backslash V^{(h,1)}_t$ is confirmed, by definition
\[
\widehat{r}^{(h)}_{t,i_0} > \widehat{r}^{(h)}_{t,a^*_{(|S^{(h)}_t|)}} + 4\cdot 2^{-h}.
\]
By the pigeonhole principle, there exists an optimal arm $j\in S^{(h)}_t$ with $\widehat{r}_{t,j}^{(h)}+w^{(h)}_{t,j} \le \widehat{r}_{t,a^*_{(|S^{(h)}_t|)}}^{(h)} + w^{(h)}_{t,a^*_{(|S^{(h)}_t|)}}$ where $a^*_{(k)}$ is the arm with the $k$-th largest $\widehat{r}^{(h)}_{t,a}+w^{(h)}_{t,a}$ by definition in \cref{alg:master_alg}. Then we have
\begin{align*}
\theta_*^\top x_{t,i_0} &\ge \widehat{r}^{(h)}_{t,i_0} - w^{(h)}_{t,i_0}\\
&\stepa{\ge} \widehat{r}^{(h)}_{t,i_0} - 2^{-h}\\
&> \widehat{r}^{(h)}_{t,a^*_{(|S^{(h)}_t|)}} + 3\cdot 2^{-h}\\
&\ge \widehat{r}^{(h)}_{t,a^*_{(|S^{(h)}_t|)}} + w^{(h)}_{t,a^*_{(|S^{(h)}_t|)}} + 2\cdot 2^{-h} \\
&\ge \widehat{r}_{t,j}^{(h)}+w^{(h)}_{t,j} + 2\cdot 2^{-h}\\
&> \theta_*^\top x_{t,j}
\end{align*}
where (a) follows since $i_0\in A_h\backslash V_t^{(h,1)}$ implies $w^{(h)}_{t,i_0}\le 2^{-h}$. This leads to $i_0\in S^{(h)}_t\subseteq V_{*,t}$ and proves claim (B). It also shows that the confirmed set $U^{(h,2)}_t\subsetneq S^{(h)}_t$ never covers all remaining optimal arms, much like \cref{lem:optimal_upper_set} in the graph-feedback case.

Next, we show that the claims (A) and (C) for stage $h$ hold conditioned on (A,B,C) at stage $h-1$. Since $i^{(h)}_{*,t} \in A_h\backslash V_t^{(h,2)}$, it is not confirmed in step (2) at stage $h-1$, so
\[
\widehat{r}^{(h-1)}_{t,i^{(h)}_{*,t}} \le \widehat{r}^{(h-1)}_{t,a_{1,(|S^{(h-1)}_t|)}} + 4\cdot 2^{-(h-1)}.
\]
Recall that $a_{1,(|S^{(h-1)}_t|)}$ is the arm with the $|S^{(h-1)}_t|$-th largest UCB $\widehat{r}^{(h-1)}_{t,a}+w^{(h-1)}_{t,a}$ at the beginning of the confirmation step (2), and $a_{1,(|S^{(h-1)}_t|)}$ is not confirmed by definition of step (2). In addition, observe that $a_{2,(|S^{(h-1)}_t|)}=a_{1,(|S^{(h-1)}_t|)}$, because every arm confirmed and selected in step (2) has a larger UCB than $a_{1,(|S^{(h-1)}_t|)}$: If $i$ is confirmed in step (2), then
\begin{align*}
\widehat{r}^{(h-1)}_{t,i} + w^{(h-1)}_{t,i} &> \widehat{r}^{(h-1)}_{t,a_{1,(|S^{(h-1)}_t|)}} + 4\cdot 2^{-(h-1)} + w^{(h-1)}_{t,i}\\
&> \widehat{r}^{(h-1)}_{t,a_{1,(|S^{(h-1)}_t|)}} + w^{(h-1)}_{t,a_{1,(|S^{(h-1)}_t|)}}.
\end{align*}
After the elimination step in (3) at stage $h-1$, we get the active set $A_h$. Recall that $w_{t,a}^{(h-1)} \le 2^{-(h-1)}$ for every arm $a\in A_h\subseteq A_{h-1}\backslash V_t^{(h-1,1)}$. Thus
\begin{align*}
\theta_*^\top x_{t,a} &\ge \widehat{r}^{(h-1)}_{t,a} - 2^{-(h-1)} 
\ge \widehat{r}^{(h-1)}_{t,a_{2,(|S^{(h-1)}_t|)}} - 3\cdot 2^{-(h-1)} \\
&= \widehat{r}^{(h-1)}_{t,a_{1,(|S^{(h-1)}_t|)}} - 3\cdot 2^{-(h-1)}  \ge \widehat{r}^{(h-1)}_{t,i^{(h)}_{*,t}} - 7\cdot 2^{-(h-1)}\\
&\ge \theta_*^\top x_{t,i^{(h)}_{*,t}} - 8\cdot 2^{-(h-1)}
\end{align*}
which proves claim (C) for stage $h$. 

To show (A) for stage $h$, conditioned on (A) at stage $h-1$, we have 
\begin{align*}
|(A_{h-1}\backslash V_t^{(h-1,3)})\cap V_{*,t}| &= |(A_{h-1}\backslash (U_t^{(h-1,1)}\cup U_t^{(h-1,2)}))\cap V_{*,t}|\\
&\ge |A_{h-1}\cap V_{*,t}| - |U_t^{(h-1,1)}\cup U_t^{(h-1,2)}|\\
&\ge S-|V_t^{(h-2,3)}| - |U_t^{(h-1,1)}\cup U_t^{(h-1,2)})|\\
&= S-|V_t^{(h-2,3)}\cup U_t^{(h-1,1)}\cup U_t^{(h-1,2)})| = S-|V_t^{(h-1,3)}|\numberthis\label{eq:step3_bound_remaining_optimal}
\end{align*}
where the first equality holds because $V_t^{(h-2,3)}$ has been removed from the active set $A_{h-1}$, and the last inequality uses (A) at stage $h-1$ (recall that $V_t^{(h-1,2)}=V_t^{(h-1,3)}$). Let $E^{(h-1)}_t$ denote the top $S-|V_t^{(h-1,3)}|$ arms in $A_{h-1}\backslash V_t^{(h-1,3)}$ (which is defined at the end of step (2) at stage $h-1$, as opposed to $S_t^{(h-1)}$ defined at the start of step (2)). By \eqref{eq:step3_bound_remaining_optimal}, $E_t^{(h-1)}\subseteq V_{*,t}$ is optimal. For any arm $i\in E^{(h-1)}_t$, by the pigeonhole principle again, there is another arm $j\in A_{h-1}\backslash V_t^{(h-1,3)}$ that satisfies $\theta_*^\top x_{t,j} \le \theta_*^\top x_{t,i}$ and $\widehat{r}_{t,j}^{(h-1)}\ge \widehat{r}^{(h-1)}_{t,a_{2,(|E_t^{(h-1)}|)}}$. Since $w^{(h-1)}_{t,a}\le 2^{-(h-1)}$ for every remaining arm $a\in A_{h-1}\backslash V_t^{(h-1,3)}$ at the time of elimination, we have
\begin{align*}
\widehat{r}^{(h-1)}_{t,i} &\ge \theta_*^\top x_{t,i} - 2^{-(h-1)} 
\ge \theta_*^\top x_{t,j}- 2^{-(h-1)} \ge \widehat{r}_{t,j}^{(h-1)} - 2\cdot 2^{-(h-1)}\\
&\ge \widehat{r}^{(h-1)}_{t,a^*_{(|E_t^{(h-1)}|)}} - 2\cdot 2^{-(h-1)}.
\end{align*}
As a consequence, this arm $i\in E^{(h-1)}_t$ is uneliminated by the end of stage $h-1$, i.e. $E^{(h-1)}_t\subseteq A_h$, which proves (A) and hence completes the induction.
\end{proof}

\subsection{Regret Upper Bound}
First, we state an elliptical potential lemma for the batched contextual bandits. It allows us to upper bound the sum of the self-normalized confidence widths defined in \cref{alg:base_alg}.

\begin{lemma}[Lemma 3 in \cite{han2020sequential}]\label{lem:batched_ellipical_potential}
Fix any collection of $d$-dimensional vectors $\bra{\bra{x_{t,\ell}}_{\ell\in[L_t]}}_{t\in[T]}$ with $\|x_{t,\ell}\|_2\le 1$. Suppose $L_t\le L$ for every $t\in[T]$ for some $L>0$. Define
\[
A_t = I + \sum_{s<t}\sum_{\ell\in[L_t]}x_{t,\ell}x_{t,\ell}^\top.
\]
It holds that
\[
\sum_{t=1}^T\sqrt{\sum_{\ell\in[L_t]}\|x_{t,\ell}\|_{A_t^{-1}}^2} \le \sqrt{10}\log(T+1)\parr*{\sqrt{dT}+ d\sqrt{L}}.
\]
\end{lemma}

Now we are in the position of proving the near-optimal regret upper bound for \cref{alg:master_alg}.

\begin{theorem}[Restatement of \Cref{thm:lin_upper_bound}]
\cref{alg:master_alg} combined with \cref{alg:base_alg} achieves
\[
\reg(\mathrm{Alg}~\ref{alg:master_alg}) = O\parr*{\log(ST)\sqrt{\log(2KT)}(\sqrt{dST}+dS)}.
\]
\end{theorem}
\begin{proof}
WLOG, we assume the event in \cref{lem:lin_width} holds for every $t\in[T]$. First consider the instantaneous regret at time $t$. Since $i_{*,t}^{(h)} = \argmin\bra{\theta_*^\top x_{t,a} : a\in A_h}$ is the optimal active arm at stage $h$, by the first claim in \cref{lem:alg_properties_lin}, $i_{*,t}^{(h)}\in V_{*,t}$. To bound the instantaneous regret, we will define a partition of the optimal $V_{*,t}$: Let $G_t^{(h)}\subseteq A_h$ denote the top $S-|V_t^{(h-1,3)}|$ arms in $A_{h}$ (as in the first claim of \cref{lem:alg_properties_lin}) and define $V_{*,t}^{(h)} = G_t^{(h)}\backslash G_t^{(h+1)}$ the optimal arms that are in the top optimal arms at stage $h$ but are not the top ones at next stage $h+1$. Observe that $G_t^{(h+1)}\subseteq G_t^{(h)}$ as the size $S-|V_t^{(h-1,3)}|$ is decreasing. By \cref{lem:alg_properties_lin}, $G_t^{(h)}\subseteq V_{*,t}$, and by definition we have $|G_t^{(h)}| = S-|V_t^{(h-1,3)}|$. Since $G_t^{(1)} = V_{*,t}$, we have 
\[
|V_{*,t}^{(1)}| = |V_{*,t}| - |G_t^{(2)}| =  S-(S-|V_t^{(1,3)}|)=|V_t^{(1,3)}|
\]
and recursively 
\[
|V_{*,t}^{(h)}|= |G_t^{(h)}| - |G_t^{(h-1)}| = |V_t^{(h,3)}| - |V_t^{(h-1,3)}| = |U_t^{(h,1)}\cup U_t^{(h,2)}\cup(\indic[h=H]U_t^{(H,3)})|,
\]
which matches the number of newly added arms to the decision $V_t$ during stage $h$. By definition, $\theta_*^\top x_{t,i_{*,t}^{(h)}} \ge \theta_*^\top x_{t,i}$ for any $i\in V_{*,t}^{(h)}\subseteq A_h$ since $i_{*,t}^{(h)}$ is the optimal active arm at stage $h$. Then we can bound the instantaneous regret using this partition:
\begin{align*}
\sum_{i\in V_{*,t}}\theta_*^\top x_{t,i} - \sum_{a\in V_t}\theta_*^\top x_{t,a} &= \sum_{h=1}^H\parr*{\sum_{i\in V_{*,t}^{(h)}}\theta_*^\top x_{t,i} - \sum_{a\in V_t^{(h,3)}\backslash V_t^{(h-1,3)}}\theta_*^\top x_{t,a}}.\numberthis\label{eq:lin_inst_reg_decompose1}
\end{align*}
Recall that the confirmed arms are optimal i.e. $U_t^{(h,2)}\subseteq V_{*,t}$ for every stage $h\in[H]$, by \cref{lem:alg_properties_lin}. 

We now show $U_t^{(h,2)}\subseteq V_{*,t}^{(h)}$: By step (2) at stage $h$, every remaining arm $a\in A_h\backslash V_t^{(h,1)}$ has a width $w^{(h)}_{t,a}\le 2^{-h}$. By pigeonhole principle, there is an arm $j\in A_h$ with $\theta_*^\top x_{t,j} \ge \theta_*^\top x_{t,i}$ and  Then if arm $i$ is confirmed in step (2), we have
\begin{align*}
\theta_*^\top x_{t,i} &\ge \widehat{r}^{(h)}_{t,i} - w^{(h)}_{t,i} 
> \widehat{r}^{(h)}_{t,a_{1,(S-|V_t^{(h,1)}|)}} + 4\cdot 2^{-h} - w^{(h)}_{t,i}\\
&\ge \widehat{r}^{(h)}_{t,a_{1,(S-|V_t^{(h,1)}|)}} + w^{(h)}_{t,a_{1,(S-|V_t^{(h,1)}|)}} + 2\cdot 2^{-h}\\
&\ge \widehat{r}^{(h)}_{t,a_{1,(S-|V_t^{(h-1,3)}|)}} + w^{(h)}_{t,a_{1,(S-|V_t^{(h-1,3)}|)}} + 2\cdot 2^{-h}\\
&\stepa{\ge} \widehat{r}^{(h)}_{t,j} + w^{(h-1)}_{t,j} + 2\cdot 2^{-h}\\
&\ge \theta_*^\top x_{t,j}+ 2\cdot 2^{-h} > \theta_*^\top x_{t,j}.
\end{align*}
In (a), by pigeonhole principle, there is an arm $j\in G_t^{(h)}$ that satisfies $\theta_*^\top x_{t,j} \ge \theta_*^\top x_{t,a_{1,(S-|V_t^{(h-1,3)}|)}}$ but $\widehat{r}^{(h)}_{t,j}+w^{(h)}_{t,j}\le \widehat{r}^{(h)}_{t,a_{1,(S-|V_t^{(h-1,3)}|)}} + w^{(h)}_{t,a_{1,(S-|V_t^{(h-1,3)}|)}}$, i.e. the expected reward is better than the benchmark, but the UCB is smaller, because $G_t^{(h)}$ is the top $S-|V_t^{(h-1,3)}|$ arms in $A_h$ in terms of expected rewards. Then $\theta_*^\top x_{t,i}>\theta_*^\top x_{t,j}$ implies $i\in G_t^{(h)}$ and thereby $U_t^{h,2}\subseteq G_t^{(h)}\backslash G_t^{(h+1)}=V_{*,t}^{(h)}$.

We can proceed to remove $U_t^{(h,2)}$ from both the optimal arm partition $V_{*,t}^{(h)}$ and the arms selected by the learners at each stage $h$ in \eqref{eq:lin_inst_reg_decompose1}. This leads to
\begin{align*}
\sum_{i\in V_{*,t}}\theta_*^\top x_{t,i} - \sum_{a\in V_t}\theta_*^\top x_{t,a} 
&= \sum_{h=1}^H\parr*{\sum_{i\in V_{*,t}^{(h)}\backslash U_t^{(h,2)}}\theta_*^\top x_{t,i} - \sum_{a\in V_t^{(h,3)}\backslash V_t^{(h-1,3)} \backslash U_t^{(h,2)}}\theta_*^\top x_{t,a}}\\
&\stepb{\le} \sum_{h=1}^H\parr*{\sum_{i\in V_{*,t}^{(h)}\backslash U_t^{(h,2)}}\theta_*^\top x_{t,i^{(h)}_{*,t}} - \sum_{a\in V_t^{(h,3)}\backslash V_t^{(h-1,3)} \backslash U_t^{(h,2)}}\theta_*^\top x_{t,a}}\\
&\le \sum_{h=1}^H\sum_{a\in U_t^{(h,1)}}\parr*{\theta_*^\top x_{t,i^{(h)}_{*,t}} - \theta_*^\top x_{t,a}} + \sum_{a\in U^{(H,0)}_t}\parr*{\theta_*^\top x_{t,i^{(H)}_{*,t}} - \theta_*^\top x_{t,a}}\\
&\stepc{\le} 16\sum_{h=1}^H\sum_{a\in U_t^{(h,1)}}2^{-h} + |U^{(H,0)}_t|\frac{16}{\sqrt{ST}}\\
&\stepd{\le} 16\sum_{h=1}^H\sum_{a\in U_t^{(h,1)}}w^{(h)}_{t,a} + 16\sqrt{\frac{S}{T}}
\end{align*}
where (b) follows from the definition that $i^{(h)}_{*,t}$ is optimal active arm in $A_h\supseteq V_{*,t}^{(h)}$, (c) applies property (C) in \cref{lem:alg_properties_lin} and the fact that every arm in $U_t^{(H,0)}$ has width $\le 2^{-H}$, and (d) is by the selection criterion of $U^{(h,1)}_t$ in \cref{alg:master_alg} such that $w^{(h)}_{t,a}>2^{-h}$ for every $a\in U^{(h,1)}_t$. In particular, by the choice of $H=\log(\sqrt{ST})$, \cref{lem:alg_properties_lin} implies
\begin{align*}
\theta_*^\top x_{t,i^{(H)}_{*,t}} - \theta_*^\top x_{t,a} \le 16\cdot 2^{-H} \le \frac{16}{\sqrt{ST}}
\end{align*}
for every $a\in U^{(H,0)}_t$.

Then the cumulative regret is bounded as follows.
\begin{align*}
\reg(\mathrm{Alg}~\ref{alg:master_alg}) &\le 16\sqrt{ST} + 16\sum_{h=1}^H\sum_{t=1}^T\sum_{a\in U_t^{(h,1)}} w^{(h)}_{t,a}\\
&\le 16\sqrt{ST} + 32(1+\beta)\sum_{h=1}^H\sum_{t=1}^T\sum_{a\in U_t^{(h,1)}}\|x_{t,a}\|_{\parr*{A_t^{(h)}}^{-1}}
\end{align*}
where $A_t^{(h)}=I + \sum_{s<t}\sum_{a\in \Phi_t^{(h)}(s)}x_{s,a}x_{s,a}^\top$ is the Gram matrix for stage $h$, as defined in \cref{alg:base_alg}. To handle the sum, note that
\begin{align*}
\sum_{a\in U_t^{(h,1)}}\|x_{t,a}\|_{\parr*{A_t^{(h)}}^{-1}} &\le \sqrt{S}\sqrt{\sum_{a\in U_t^{(h,1)}}\|x_{t,a}\|_{\parr*{A_t^{(h)}}^{-1}}^2}\\
&\le \sqrt{S}\sqrt{\sum_{a\in \Phi_{T+1}^{(h)}(t)}\|x_{t,a}\|_{\parr*{A_t^{(h)}}^{-1}}^2}
\end{align*}
where the first inequality is by Cauchy-Schwartz inequality and that $|U_t^{(h,1)}|\le S$, and the second follows from $U_t^{(h,1)}\subseteq \Phi_{T+1}^{(h)}(t)$ as $U_t^{(h,1)}$ is added to the index set at step (1) in \cref{alg:master_alg}. Finally, we have
\begin{align*}
\reg(\mathrm{Alg}~\ref{alg:master_alg}) &\le 16\sqrt{ST} + 32(1+\beta)\sqrt{S}\sum_{h=1}^H\sum_{t=1}^T\sqrt{\sum_{a\in \Phi_{T+1}^{(h)}(t)}\|x_{t,a}\|_{\parr*{A_t^{(h)}}^{-1}}^2}\\
&\stepc{\le} 16\sqrt{ST} + c_0\log(ST)\sqrt{\log(2KT)}\sqrt{S}\parr*{\sqrt{dT} + d\sqrt{S}}
\end{align*}
where (c) follows from \cref{lem:batched_ellipical_potential}, with some absolute constant $c_0>0$.
\end{proof}

\subsection{Minimax Lower Bound}
\begin{theorem}[Restatement of \cref{thm:lin_lower_bound}]
Suppose $K\ge 2S$ and $T\ge \max\bra{4dS, \frac{d^3}{S}}$. For any policy $\pi$, it holds that
\[
\max_{\theta_*, \bra{x_{t,a}}}\reg(\pi) = \Omega(\sqrt{dST})
\]
where the maximum is taken over all problem instances as described in \cref{sec:lin_context_intro}.
\end{theorem}
\begin{proof}
Fix any policy $\pi$. We will prove the lower bound by a reduction argument to the combinatorial bandits under graph feedback. 

\textbf{($d>2S$).} First suppose $d> 2S$. WLOG, let $d=2nS+1$ for some $n\ge 1$. We partition the horizon $[T]$ into $n$ sub-horizons of equal length $\frac{T}{n}$ (WLOG, suppose $\frac{T}{n}$ is an integer). Specifically, define
\[
t_1=1,\quad t_{j+i} = t_j + \frac{T}{n}
\]
and the $j$-th sub-horizon be $T_j=[t_j, t_{j+1})$. 

Let $e(k)\in\R^d$ denote the one-hot canonical basis $[0, \dots, 0, \underbrace{1}_{k\text{-th}}, 0, \dots, 0]^\top$. The context sequence over sub-horizon $T_j$ is defined as follows: for arm $a\in[K]$, let $x_{t,a}=\frac{1}{\sqrt{2}}\parr*{e(1)+e(1+(j-1)2S + a)}$ for all $t\in T_j$. Then during $T_j$, the expected reward of every arm $a$ is $\theta_*^\top x_{t,a} = \frac{1}{\sqrt{2}}\theta_{*, 1} + \frac{1}{\sqrt{2}}\theta_{*, 1+(j-1)2S+a}\eqqcolon \frac{1}{\sqrt{2}}\theta_{*, 1} + \Delta_{a}^{(j)}$ which is a constant over this sub-horizon, i.e.
\[
\theta_* = \parq*{\theta_{*, 1}, \underbrace{\Delta_{1}^{(1)},\dots, \Delta_{2S}^{(1)}}_{\text{relevant during }T_1}, \underbrace{\Delta_{1}^{(2)},\dots, \Delta_{2S}^{(2)}}_{\text{relevant during }T_2}, \Delta_{1}^{(3)},\dots, \Delta_{2S}^{(n)}}.
\]
The reward $r_{t,a}\sim\mathrm{Bern}(\theta_*^\top x_{t,a})$ is drawn from a Bernoulli distribution.

Consequently, the regret of policy $\pi$ can be decomposed into the sum of regrets of $n$ independent sub-problems over the $n$ sub-horizons. During each $T_j$, the learning problem now becomes a combinatorial semi-bandit with Bernoulli rewards. Following the lower bound construction in Theorem 1.3 of \cite{wen2025adversarial}, for any $j\in[n]$, there is an instance with $\theta_{*,1}=\frac{1}{2\sqrt{2}}$ and $0\le \Delta_{a}^{(j)} \le \frac{1}{64}\sqrt{\frac{2S}{ST}}=\frac{1}{64}\sqrt{\frac{2}{T}}$ such that the regret of policy $\pi$ over $T_j$ is $\Omega(\sqrt{(2S)S|T_j|})$, subject to $|T_j|\ge \frac{(2S)^3}{S}$. The constraint $|T_j|=\frac{(2S)T}{d-1}\ge \frac{(2S)^3}{S}$ is satisfied since $T\ge 4Sd$.
Then the total regret over the $n$ sub-horizons is bounded by
\[
\reg(\pi) \ge \sum_{j=1}^n\Omega\parr*{\sqrt{S^2|T_j|}} = \Omega\parr*{\frac{d-1}{S}\sqrt{S^2\frac{ST}{d-1}}} = \Omega\parr*{\sqrt{(d-1)ST}} = \Omega\parr*{\sqrt{dST}}
\]
when $d>2S$. We can verify that the parameter norm is
\[
\|\theta_*\|_2^2 = \frac{1}{8} + \frac{d-1}{2048}\frac{1}{T} \le 1
\]
as $T\ge 4Sd \ge \frac{d}{1024}$.

\textbf{($1<d\le 2S$).} 
% \YH{$K$ here should be $2S$?} 
Now suppose $1<d\le 2S$. We instead consider the entire horizon $[T]$ with the following construction. For arm $a\in[d-1]$, let $x_{t,a} = \frac{1}{\sqrt{2}}(e(1) + e(1+a))$ so its expected reward is $\theta_*^\top x_{t,a} = \frac{1}{\sqrt{2}}\theta_{*, 1} + \frac{1}{\sqrt{2}}\theta_{*, 1+a}$. The reward $r_{t,a}\sim\mathrm{Bern}(\theta_*^\top x_{t,a})$ still follows a Bernoulli distribution. For $a\ge d$, let $r_{t,a}\equiv 0$ by setting $x_{t,a}\equiv 0$. Note that this problem reduces to a combinatorial bandit with graph feedback, and the feedback graph $G$ includes edge $a\to a'$ only when both $a,a'>d-1$ (which share the same mean reward $0$). Following the construction in Theorem 1.3 of \cite{wen2025adversarial} again gives a regret lower bound $\Omega(\sqrt{(d-1)ST}) = \Omega(dST)$ subject to $T\ge \frac{(d-1)^3}{S}$.\footnote{Note that the lower bound $\Omega(S\sqrt{T})$ in \cite{wen2025adversarial} does not appear in our case. For this part of the lower bound, \cite{wen2025adversarial} directly adopts the full information lower bound $\Omega(S\sqrt{T})$ from \cite{koolen2010hedging}, whose construction crucially relies on the fact that no two arms share the same reward mean (at least not to the learner's knowledge). This is however not the case here, as the learner can cluster these one-hot contexts and identify which arms share the same mean. If $n$ arms are known to share the same mean or share the mean up to some known scalars, there are roughly $n$ times more observations for each of those arms, leading to an $n$ times faster concentration.}

\textbf{($d=1$).} Finally, consider the special case $d=1$. Since $K\ge 2S$, we can partition the arms into two subsets $V_+$ and $V_0$, with $|V_+|\ge S$ and $|V_0|\ge S$. For $a\in V_+$, let the context be $x_{t,a}\equiv \Delta$ for some $\Delta\in(0,\frac{1}{2}]$ to be specified later, and let $x_{t,a}\equiv 0$ for $a\in V_0$. For each arm $a$, the reward is drawn from a scaled Bernoulli such that
\[
r_{t,a} = \begin{cases}
    1 &\text{with probability $\frac{1+\theta_* x_{t,a}}{2}$}\\
    -1 &\text{otherwise}
\end{cases}.
\]
We will consider two environments indexed by $\theta_*\in\bra{-1, 1}$. Since the contexts are time-invariant, the expected reward for arm $a$ is $\mu_a = \theta_* x_{t,a} = \indic[a\in V_+]\theta_*\Delta$. Let $\mathbb{P}_+$ (resp. $\mathbb{P}_-$) denote the law of the policy $\pi$ under environment $\theta_*=1$ (resp. $\theta_*=-1$), and the expectations $\E_+$ and $\E_-$ respectively under each law.

The regret under environment $\theta_*=1$ is denoted by $\reg_+(\pi) \ge \Delta \E_+\parq*{N_0}$, and the regret under environment $\theta_*=-1$ is $\reg_-(\pi) \ge \Delta \E_-\parq*{N_+}$, where $N_0=\sum_{t=1}^T|V_t\cap V_0|$ and $N_+=\sum_{t=1}^T|V_t\cap V_+| = ST-N_0$ count the number of pulls in the set $V_0$ and $V_+$ respectively. Let $\theta_*$ be drawn uniformly from $\bra{-1,1}$. The minimax regret is lower bounded by the Bayes regret
\begin{align*}
\max_{\theta_*, \bra{x_{t,a}}}\reg(\pi) &\ge \frac{1}{2}\parr*{\reg_+(\pi) + \reg_-(\pi)}\\
&\ge \frac{\Delta}{2}\parr*{ST- \E_+[N_+] + \E_-[N_+]}.\numberthis\label{eq:lower_bound_bayes}
\end{align*}
Let $R^T=\bra{r_{t,a}: t\in[T], a\in[K]}$ denote the generated reward sequence. By Pinsker's inequality and the chain rule of KL divergence, we have
\begin{align*}
\E_+[N_+] - \E_-[N_+] &\le ST\sqrt{\frac{1}{2}\mathsf{KL}\parr*{\mathbb{P}_-(R^T)\|\mathbb{P}_+(R^T)}}\\
&=ST\sqrt{\frac{1}{2}\sum_{t=1}^T\sum_{R^{t-1}}\mathbb{P}_-(R^{t-1})\mathsf{KL}\parr*{\mathbb{P}_-(r^t|R^{t-1})\|\mathbb{P}_+(r^t|R^{t-1})}}\numberthis\label{eq:lower_bound_N_diff}
\end{align*}
where $r^t = \bra{r_{t,a}: a\in[K]}$ is the generated rewards at time $t$. Since the feedback is semi-bandit, the observed distribution only makes a difference when $V_t\cap V_+ \neq \varnothing$. Denote $V_t=\bra{v_{t,1},\dots, v_{t,S}}$. Thus we have
\begin{align*}
\sum_{t=1}^T\sum_{R^{t-1}}\mathbb{P}_-(R^{t-1})\mathsf{KL}\parr*{\mathbb{P}_-(r^t|R^{t-1})\|\mathbb{P}_+(r^t|R^{t-1})} &= \sum_{t=1}^T\sum_{s=1}^S\mathbb{E}_-[\indic[v_{t,s}\in V_+]]\mathsf{KL}\parr*{\mathrm{Bern}\parr*{\frac{1}{2}}\|\mathrm{Bern}\parr*{\frac{1+\Delta}{2}}}\\
&\stepa{\le} \sum_{t=1}^T\sum_{s=1}^S\mathbb{E}_-[\indic[v_{t,s}\in V_+]]\frac{4\Delta^2}{3}\\
&\le \frac{4}{3}\Delta^2 ST
\end{align*}
where (a) uses the inequality $\mathsf{KL}\parr*{\mathrm{Bern}(p)\|\mathrm{Bern}(q)} \le \frac{(p-q)^2}{q(1-q)}$ and $\Delta\in(0,\frac{1}{2}]$. Plugging back into \eqref{eq:lower_bound_N_diff} gives
\[
\E_+[N_+] - \E_-[N_+] \le ST\sqrt{\frac{4}{3}\Delta^2ST}.
\]
Then we have the lower bound in \eqref{eq:lower_bound_bayes} to be
\begin{align*}
\max_{\theta_*, \bra{x_{t,a}}}\reg(\pi) &\ge \frac{\Delta}{2}\parr*{ST - ST\sqrt{\frac{4}{3}\Delta^2ST}} = \frac{\Delta ST}{2}\parr*{1-\sqrt{\frac{4}{3}\Delta^2ST}}.
\end{align*}
Take $\Delta = 1/\sqrt{12ST}\in(0,\frac{1}{2}]$, we have $\max_{\theta_*, \bra{x_{t,a}}}\reg(\pi) \ge \frac{\sqrt{ST}}{3}$. This concludes the proof.
\end{proof}

\section{Constrained Decision Subsets in Section \ref{sec:constrained_subset}}\label{app:constrained_subset}

As promised in \cref{sec:constrained_subset}, we consider a potentially constrained decision subset $\Acal_0\subseteq \binom{[K]}{S}$ and another elimination-based algorithm.

\begin{algorithm}[h!]\caption{Combinatorial Arm Elimination \citep{wen2025adversarial}}
\label{alg:comb_arm_elim}
\textbf{Input:} time horizon $T$, decision subset $\Acal_0\subseteq\binom{[K]}{S}$, arm set $[K]$, combinatorial budget $S$, feedback graph $G$, and failure probability $\delta\in(0,1)$.

\textbf{Initialize:} Active set $\Aact\gets \Acal_0$, minimum count $N\gets 0$, time $t\gets 1$.

Let $(\br_{t,a}, n_{t,a})$ be the empirical reward and the observation count of arm $a\in[K]$ at the end of time $t$.

For each combinatorial decision $V\in\Aact$, let $\br_{t,V} = \sum_{a\in V}\br_{t,a}$ be the empirical reward and $n_{t,V} = \min_{a\in v}n_{t,a}$ be the its observation count.

\While{$t\le T$}{
Let $\Acal_N\gets \bra*{V\in \Aact: n_{t,V} = N}$ be the decisions that have been observed least.

Let $U_t=\bra*{a\in[K]: \exists V\in\Acal_N \textnormal{ with } a\in V} = \bigcup_{V\in \Acal_N}V$.

Find a minimal set of decisions $\bra{V_1,\dots, V_n}\subseteq \Acal_0$ such that $U_t\subseteq \cup_{m=1}^n\Nout(V_m)$.

\For{$m=1$ \KwTo $n$}{
Select decision $V_m$.

Collect feedback $\bra*{r_{t,a}: a\in\Nout(V_m)}$ and update $(\br_{t,a}, n_{t,a})$ accordingly.

Update time $t\gets t+1$.
}

Note by the choice of $\bra{V_1,\dots, V_n}$, now $\min_{V\in\Aact} n_{t,V} > N$.

Update the minimum count $N\gets \min_{V\in\Aact}n^t_V$.

Let $\br_{t,\max}\gets \max_{V\in\Aact}\br_{t,V}$ be the maximum empirical reward in the active set.

Update the active set as follows:
\[
\Aact \gets \bra*{V\in\Aact : \br_{t,V} \geq \br_{t,\max} - 2S\sqrt{\frac{\log(2KT/\delta)}{N}}}.
\]

}
\end{algorithm}

Before presenting the regret guarantee, we remark that finding a minimal set of decisions $\bra{V_1,\dots, V_n}$ that covers the set $U_t$ in \cref{alg:comb_arm_elim} can be computationally hard in general. Therefore, the final algorithm can be inefficient.

\begin{theorem}[Restatement of \Cref{thm:general_subset_upper}]
Fix any $\delta\in(0,1)$. With probability at least $1-\delta$,
\[
\reg(\mathsf{Alg}~\ref{alg:comb_arm_elim}) = O\parr*{\log(2KT/\delta)\frac{S^2\kappa}{\Delta_*}}
\]
where $\kappa$ is defined as in \eqref{eq:kappa} and $\Delta_*$ is the reward gap between the optimal and the second optimal decisions in $\Acal_0$.

In particular, there is a subset $\Acal_0$ under which
\[
\min_\pi\max_\nu\reg_\nu(\pi) = \Omega\parr*{\log(T)\frac{S^2\kappa}{\Delta_*}}
\]
where the maximum is taken over all bandit environments $\nu$.
\end{theorem}
\begin{proof}
WLOG, suppose the confidence bound in \cref{lem:sto_reward_concentration} holds. Let $N^t$ denote the minimum count $N$ at the end of time $t$, and $w(n)\coloneqq \sqrt{\log(2KT/\delta)/n}$ denote the confidence width. For a decision $V\in\Acal_0$, define its expected reward as $\mu_V = \sum_{a\in V}\mu_a$. Then we have the following two observations:
\begin{itemize}
    \item[(1)] The optimal decision $V_*\in\Aact$ always remains uneliminated. This follows from
    \[
    \br_{t,V_*} \ge \mu_{V_*}  - Sw(N^t) \ge \mu_V - Sw(N^t) \ge \br_{t,V} - 2Sw(N^t)
    \]
    for any other decision $V\in\Aact$.

    \item[(2)] At any time $t$, any active decision $V\in\Aact$ satisfies $\mu_{V_*} - \mu_V\le 4Sw(N^t)$. This follows from
    \[
    \mu_{V_*} \le \br_{t,V_*} + Sw(N^t) \le \br_{t,V} + 3Sw(N^t) \le \mu_V + 4Sw(N^t)
    \]
    for any $V\in\Aact$ by the elimination criterion.
\end{itemize}

Define a layer $L_n\coloneqq \bra{V\in \Acal_0: \text{ V is selected when the minimum count is $N^t=n$}}\subseteq \Acal_0 \subseteq \binom{[K]}{S}$. Note that when $4Sw(N^t)<\Delta_*$, the active set contains only the optimal decision, so
\[
\max\bra{n\ge 1: |L_n|>1} \le \frac{16S^2\log(2KT/\delta)}{\Delta_*^2} \eqqcolon M.
\]
We have $|L_n|\le \kappa$ by the choice of $V_t$ in \cref{alg:comb_arm_elim}. Then the cumulative regret is bounded as
\begin{align*}
\reg(\mathsf{Alg}~\ref{alg:comb_arm_elim}) &\le S|L_0| + \sum_{n= 1}^M\sum_{V\in L_n}\parr*{\mu_{V_*}-\mu_V}\\
&\le S\kappa + 4S\sum_{n=1}^{M}|L_n|w(n)\\
&\le S\kappa + 8S\kappa\sqrt{\log(2KT/\delta)}\sqrt{M}\\
&= S\kappa + 32\log(2KT/\delta)\frac{S^2\kappa}{\Delta_*}.
\end{align*}
For the second claim, \cite{kveton2015tight} shows the following in their Proposition 1. Under a subset $\Acal_0$ that partitions the $K$ arms, i.e. $\forall V_1,V_2\in\Acal_0, V_1\cap V_2=\varnothing$ and $\sum_{V\in\Acal_0}|V|=\Omega(K)$, the minimax regret is $\Omega(\log(T)KS/\Delta_*)$. Under this subset, we have $\kappa=\Omega(K/S)$ and hence $S^2\kappa=\Omega(SK)$ recovers the minimax lower bound.
\end{proof}

\section{Auxiliary Lemmata}
We first define the \textit{dominating number} as a relevant graph-theoretic quantity. For any directed graph $G=(V,E)$, the dominating number is defined by the cardinality of the smallest subset whose out-neighbor covers the entire graph:
\begin{equation}\label{eq:def_dom_number}
\delta(G) \coloneqq \min\bra{|D|: D\subseteq V\text{ and }V\subseteq \Nout(D)}.
\end{equation}

For any directed graph $G=(V,E)$, consider the following greedy approximation algorithm that finds a dominating subset: one starts with an empty set $D\gets \varnothing$, recursively finds an ``unobserved" node $v\notin\Nout(D)$ that has the largest out-degree in the remaining subgraph on $V\backslash\Nout(D)$, and add it to the set $D\gets D\cup\bra{v}$. The following result characterizes the performance of this well-known approximation.

\begin{lemma}[\cite{chvatal1979greedy}]\label{lem:greedy_dominating_subset}
For any directed graph $G=(V,E)$, the above greedy algorithm gives a dominating subset $D\subseteq V$ that satisfies
\[
|D| \le (1+\log|V|)\delta(G).
\]
\end{lemma}

\begin{lemma}[Lemma 8 in \cite{alon2015online}]\label{lem:alon15}
For any directed graph $G=(V,E)$, it holds that $\delta(G) \le 50\log|V|\alpha(G)$.
\end{lemma}

\end{document}